\documentclass[11pt]{article}

\usepackage{fullpage}
\usepackage[small]{caption}
\usepackage{fancyhdr}
\usepackage{natbib}
\usepackage{titling}

\usepackage{amssymb,amsmath,amsfonts,amsthm,algorithm,algpseudocode}
\usepackage{graphicx}
\graphicspath{{./}}
\usepackage{booktabs}
\usepackage{float}
\usepackage{tikz}
\usetikzlibrary{fadings, arrows.meta, calc, shapes.geometric}

\newtheorem{theorem}{Theorem}
\newtheorem{proposition}{Proposition}
\newtheorem{lemma}{Lemma}
\newtheorem{corollary}{Corollary}
\theoremstyle{definition}
\newtheorem{definition}{Definition}
\theoremstyle{remark}
\newtheorem{remark}{Remark}

\DeclareMathOperator*{\argmax}{arg\,max}

\usepackage{xcolor}
\usepackage[hidelinks]{hyperref}

\title{Algorithmic Impact Reveals\\the Hidden Social Choice Structure of Alignment}

\newcommand{\authorbox}[1]{\makebox[16em][c]{\begin{tabular}[t]{c}#1\end{tabular}}}
\renewcommand{\And}{\hspace{4em}}
\newcommand{\AND}{\\ \noalign{\vspace{1.25em}}}

\author{%
  \authorbox{%
    \textbf{Zachary Wojtowicz}\\
    MIT \\
    Cambridge, MA 02139 \\
    {\texttt{zachwoj@mit.edu}}%
  }
  \And
  \authorbox{%
    \textbf{Michelle Si}\\
    Harvard University \\
    Cambridge, MA 02138 \\
    {\texttt{msi@g.harvard.edu}}%
  }
  \AND
  \authorbox{%
    \textbf{Finale Doshi-Velez} \\
    Harvard University \\
    Cambridge, MA 02138 \\
    {\texttt{finale@seas.harvard.edu}}%
  }
  \And
  \authorbox{%
    \textbf{Ariel Procaccia} \\
    Harvard University \\
    Cambridge, MA 02138 \\
    {\texttt{arielpro@seas.harvard.edu}}%
  }
}
\date{}

\begin{document}

\maketitle

\begin{abstract}

  When an AI algorithm makes decisions that affect more than one person, aligning it becomes a problem of social choice: how should people's divergent preferences about system behavior be reconciled and aggregated into a single coherent model?
  The standard approach to aligning frontier AI models---reinforcement learning from human feedback---largely sidesteps this question and has poor social choice guarantees. However, it remains unclear what alternative should replace it. 
  We show that, when agent utilities can be expressed linearly in a fixed feature representation, the alignment problem can be reformulated as linear optimization over a convex \textit{impact space}, which makes it amenable to the standard toolkit of welfare economics and mechanism design. 
  This reformulation clarifies how alignment protocols translate into welfare consequences and, conversely, how a social planner's desired constraints on welfare consequences can be translated back into alignment protocols. 
  We apply this transformation to show that voting-by-issues and random-dictatorship mechanisms are strategyproof and unanimous.
  Demonstrating the reverse direction, we also apply the impact representation to derive a family of alignment protocols that maximize utilitarian social welfare subject to various social desiderata, such as bounds on individual or group harm.
  We illustrate the welfare implications of these alignment protocols empirically using real human preferences over kidney allocation, charitable food distribution, LLM responses, and trolley problems.\footnote{All code can be found at \url{https://github.com/michelleeesi/hiddenstructure}.}

\end{abstract}

\section{Introduction}
\label{sec:intro}

Alignment is typically framed as the problem of ensuring that an AI system respects human preferences. 
However, AI systems increasingly take actions that affect more than one person, and some disagreement about how these systems should behave is inevitable. 
In such cases, alignment becomes a problem of \emph{social choice}: how should an AI model combine the potentially conflicting preferences of many individuals to select among alternatives? 

Recently, \emph{reinforcement learning from human feedback (RLHF)} has become the standard approach to aligning AI systems~\citep{bai2022training}. 
In this method, human annotators provide structured feedback on model responses, typically in the form of pairwise comparisons, which are then pooled together and used to train a model, either directly or through a proxy reward model. 
Although widely used, this approach does not directly engage with the social choice problem intrinsic to alignment, and recent work has shown that RLHF produces models that have a variety of potentially undesirable properties, such as the down-weighting of minority groups~\citep{ge2024axioms,golz2025distortion,shirali2025direct,chakraborty2024maxmin, chidambaram2024direct, siththaranjan2023distributional, halpern2025pairwise}. 

These limitations have prompted recent calls to elicit preferences at a more granular level, such as the group or individual, rather than pooling all feedback anonymously~\citep{park2024rlhf, shirali2025direct, li2024personalized, poddar2024personalizing, dai2024mapping}. However, this leaves open the question: how should individual preferences, once collected, be combined? Social choice theory was developed precisely to guide such decisions~\citep{conitzer2024social}. Yet applying classical results from the field to alignment has been complicated by the structure of modern AI systems, in which alternatives are high-dimensional model parameterizations rather than discrete options. The key insight of our paper is that, under the assumed linear utility model, working in impact space removes this parameterization complexity:
\begin{quote}
    \emph{A decision-making algorithm can be summarized by a single impact vector that captures how its choices affect individual welfare during deployment. Once reformulated in impact space, aligning a model to maximize utilitarian social welfare reduces to optimizing a linear function over a convex polytope.}
\end{quote}
Our central new result, Theorem~\ref{thm:welfare-geometry}, connects model parameters to welfare impacts and shows how to move between them. This lets us apply existing results from \emph{linear social choice}~\citep{ge2024axioms} to alignment.

\textbf{Contributions:} Theorem~\ref{thm:welfare-geometry} in Section~\ref{sec:framework} is our main contribution, which shows that maximizing social welfare in the linear social choice setting is a linear optimization problem over a convex polytope that we call the impact space. In Section~\ref{sec:strategy}, we show that in impact space strategyproof alignment mechanisms inherit the classical menu structure of \textit{option-set} mechanisms (Lemma~\ref{thm:option-set}, after \citealp{barbera_peleg_1990}), and we show that aggregating each deployment query by a monotone vote yields the strategyproof \textit{voting-by-issues} rules (Theorem~\ref{thm:voting-by-issues}), with \textit{random dictatorships} as a key example. We also establish that random dictatorships can be implemented, in expected impact, by a single model parameter (Theorem~\ref{thm:single-sp}). In Section~\ref{sec:constrained-welfare}, we use our framework to characterize a family of mechanisms that maximize welfare subject to constraints on individual and group outcomes (Theorem~\ref{thm:constrained-linear-social-choice}), such as bounds on harm (of which max-min protection of the worst-off agent is a special case, Corollary~\ref{cor:bounded-harm}) and public-spirited trade-offs between personal loss and public gain (Corollary~\ref{cor:public-spirit}). These applications highlight how working in impact space makes it easy to build social objectives beyond utilitarian welfare maximization into AI systems. Finally, in Section~\ref{sec:empirical} we illustrate the welfare consequences of these alignment protocols using real human preferences from four domains: kidney allocation, charitable food distribution, LLM responses, and trolley problems.

\section{Related Work}\label{sec:related}

Our work builds on a recent literature that studies the distributional welfare implications and other social properties of RLHF, particularly for populations with diverse preferences.  

\textbf{Characterizing RLHF Distortion.} Understanding the shortcomings of RLHF has become an area of heightened interest \citep{siththaranjan2023distributional, shirali2025direct, golz2025distortion}.
\cite{golz2025distortion} measure how suboptimal (relative to a utilitarian optimum) an alignment method can be in aggregate, while we focus on what generates that suboptimality on a participant-to-participant (or group-to-group) basis. \cite{shirali2025direct} show that pooled RLHF differs from using the average of individual choice probabilities by a variance term informed by preference dispersion.

\textbf{Strategyproof Alignment.} A separate line of work studies strategic behavior in preference collection \citep{sun2024mechanism, kleinebuening2025strategyproof}. The main difference between \cite{kleinebuening2025strategyproof} and our paper is the strategyproof object. \cite{kleinebuening2025strategyproof} study the finite-sample RLHF learning pipeline, where labelers manipulate noisy pairwise labels; they propose an approximately strategyproof learner, the Median-of-MLEs. We abstract away preference estimation and focus on the downstream social-choice problem of aggregating reported linear preferences, which lets us characterize strategyproof mechanisms directly in impact space, with random dictatorship as the focal example. Thus, the two papers address different layers of the problem---algorithmic robustness vs.\ mechanism structure.

\textbf{Alternative Alignment Protocols.} Other papers propose algorithmic alternatives for heterogeneous preferences~\citep{chakraborty2024maxmin, chidambaram2024direct, park2024rlhf}, from MaxMin-RLHF to personalization. These papers optimize over latent preference types given particular desiderata; we instead show how utilitarian, strategyproof, harm-bounded, and public-spirit alignment are all optimization problems over the same set, the impact space. Our work is most closely related to the linear social choice framework of \citet{ge2024axioms}. Theorem~\ref{thm:welfare-geometry} formally establishes that, when utilities are linearly representable, alignment is an instance of linear social choice.

\section{Preliminaries}
\label{sec:prelim}

\subsection{Choice in the Linear Utility Setting}

We begin with the standard model of stochastic choice assumed by \emph{Direct Preference Optimization} \citep{rafailov2023direct} and related methods. Let $X$ be a finite set of contexts, $Y$ a finite set of actions, and $N$ the number of agents. Agent $n$ has utility $u_n:X\times Y\to\mathbb R$. For a binary choice problem $(x,y,y')$, we assume stochastic choice follows a Bradley--Terry--Luce model
$$
    p_n(y\succ y'\mid x,y,y')
    =
    \sigma\!\left(u_n(y\mid x)-u_n(y'\mid x)\right),
\qquad
    \sigma(a)=\frac{\exp(a)}{1+\exp(a)}.$$ Each context--action pair is embedded as $\phi(x,y)\in\mathbb R^k$. Our central assumption, standard in theoretical analyses of RLHF \citep{zhu_jordan_jiao_2023,nika_etal_2024,ge2024axioms,kleinebuening2025strategyproof,barreto_etal_2026}, is that each agent's utility can be expressed linearly in this feature space: $
    u_n(y\mid x)=\phi(x,y)^\top\theta_n$ for $
    \theta_n\in\mathbb R^k .$
Linearity is in the chosen feature map $\phi$, not in the raw attributes of an outcome; utility that is nonlinear in those attributes is accommodated by enriching $\phi$, and everything below applies to the enriched map unchanged (Appendix~\ref{app:linearity} tests this on our data).
For social welfare weights $w\in\Delta_N$, define the utilitarian preference parameter $\bar\theta_w=\sum_n w_n\theta_n$. By linearity, the $w$-weighted utilitarian aggregate is represented by the same feature map:
$$
    \bar u_w(y\mid x)
    =
    \sum_n w_n u_n(y\mid x)
    =
    \phi(x,y)^\top \bar\theta_w .
$$
Restricting attention to weighted sums of individual utilities is itself justified by the aggregation theorem of \citet{harsanyi1955cardinal}, which shows that a social preference satisfying the expected-utility axioms over a convex set of prospects, together with Pareto indifference, must take exactly this form.

Under linear utilities, each pairwise comparison depends only on the feature difference
$
    \alpha(x,y,y')=\phi(x,y)-\phi(x,y') .
$
We write $\alpha_z\in\mathbb R^k$ for the vector associated with a choice problem $z\in Z=X\times Y\times Y$. Let $q_{\mathrm{train}} \in \Delta([N] \times Z)$ be the distribution of training query-agent pairs, and $q_{\mathrm{dep}} \in \Delta(Z)$ be the distribution of deployment problems. 

\subsection{Deployment Welfare}

A central objective for alignment is ensuring that model behaviors improve aggregate social welfare. We measure the welfare agent $n$ receives from a model deployed at parameter $\theta$ as the amount of utility it generates relative to the baseline of choosing randomly between the two options,
\begin{equation*}
    U_n(\theta)
    =
    \mathbb E_{z\sim q_{\mathrm{dep}}}
    \left[
        \alpha_z^\top\theta_n
        \left(
            \sigma(\alpha_z^\top\theta)-\frac{1}{2}
        \right)
    \right].
\end{equation*}

The utilitarian social welfare of a deployed model is then the weighted sum of the individual welfares, $U(\theta)=\sum_{n=1}^N w_n U_n(\theta)$. 


\section{Framework: Alignment as Linear Optimization in Impact Space}
\label{sec:framework}


A natural remedy to the challenge of social alignment is to estimate preferences at the individual or group level and then combine these estimates into a single deployed model.
But this immediately raises the social choice questions that anonymous RLHF sidesteps: which aggregation rule should be used? What welfare guarantees does it provide? Can agents manipulate it?
Answering them in model-parameter space is hard because welfare is a nonlinear function of the parameters.

\emph{The key insight of our framework is that this nonlinearity can be avoided by working directly with the set of potential welfare consequences rather than with the parameterization.}

Under our linear utility assumptions, a model's welfare effects can be summarized by a single  vector---its \emph{impact}---that lives in the same space as preferences and deployment queries (see Figure \ref{fig:overall}). Welfare is linear in this vector, so the social alignment problem reduces to optimizing a linear function over a convex set: a problem for which we have the tools of welfare economics and mechanism design at our disposal.

\begin{definition}
The \textbf{impact} of a model with parameter $\theta$ is $\psi(\theta)=\mathbb E_{z \sim q_{\mathrm{dep}}}\left[\alpha_z\left(\sigma(\alpha_z^\top \theta)-\tfrac{1}{2}\right)\right]\in\mathbb R^k$.
\end{definition}

The impact vector is a sufficient statistic for welfare.  Since $U_n(\theta) = \mathbb{E}_{z\sim q_{\mathrm{dep}}}[\alpha_z^\top\theta_n (\sigma(\alpha_z^\top\theta) - \frac{1}{2})] = \theta_n^\top\psi(\theta)$, agent $n$'s welfare is the inner product of their preference with the impact 
vector.\footnote{Economically, $\psi$ is analogous to a vector of quantities delivered and $\theta_n$ encodes agent $n$'s marginal valuations.} Utilitarian welfare is therefore also linear:  $U(\psi) = \bar{\theta}_w^\top\psi$. 

\begin{figure}[t]
    \centering
    \tikzfading[name=mist, top color=transparent!100, bottom color=transparent!0]
    \resizebox{0.585\linewidth}{!}{%
    \begin{tikzpicture}[
        scale=2, >=Latex, font=\scriptsize,
        pref/.style={dashed, line width=0.9pt, -{Latex[length=1.9mm]}},
    ]

    \definecolor{cA}{HTML}{2A78D6}         
    \definecolor{cB}{HTML}{E34948}         
    \definecolor{cU}{HTML}{2E2E2E}         
    \definecolor{psiedge}{RGB}{128,145,172} 
    \definecolor{psifill}{RGB}{236,240,246}

    \path[use as bounding box] (-1.55,-0.27) rectangle (1.65,1.26);
    \clip (-1.55,-0.27) rectangle (1.65,1.30);

    \coordinate (v1) at (-0.70, 0.70);   
    \coordinate (v2) at ( 0.10, 0.82);   
    \coordinate (v3) at ( 1.05, 0.40);   
    \coordinate (v4) at ( 0.70,-0.70);
    \coordinate (v5) at (-0.10,-0.82);
    \coordinate (v6) at (-1.05,-0.40);   
    \coordinate (O)  at ( 0,0);

    \fill[psifill, draw=psiedge, line width=0.7pt]
      (v1)--(v2)--(v3)--(v4)--(v5)--(v6)--cycle;
    \node[psiedge!45!black] at (-0.62,0.16) {$\bar\Psi$};

    \draw[black!85, line width=0.5pt] (-1.45,0)--(1.42,0);
    \draw[black!85, line width=0.5pt] (0,-0.95)--(0,1.25);

    \fill[white, path fading=mist] (-1.55,-0.27) rectangle (1.65,-0.03);

    \draw[black, dashed, line width=0.8pt] (v1) -- (v3);

    \draw[cA, opacity=1, line width=1.0pt] (-0.470,0.822) -- (-0.930,0.578);  
    \draw[cB, opacity=1, line width=1.0pt] ( 1.214,0.198) -- ( 0.886,0.602);  
    \draw[cU, opacity=1, line width=1.0pt] ( 0.348,0.741) -- (-0.148,0.899);  

    \fill[cA] (v1) circle (0.85pt);
    \fill[cB] (v3) circle (0.85pt);
    \node[star, star points=5, star point ratio=2.3, fill=cU, draw=cU,
          minimum size=6pt, inner sep=0pt] at (v2) {};

    \draw[pref, cA] (O) -- (-0.563,1.060);  
    \draw[pref, cB] (O) -- ( 1.213,0.984);  
    \draw[pref, cU] (O) -- ( 0.325,1.022);  

    \node[diamond, fill=cU, draw=cU, minimum size=4pt, inner sep=0pt] at (0.175,0.55) {};

    \fill[cU] (O) circle (0.9pt);
    \node[below left=-1.5pt] at (O) {$O$};

    \node[cA]   at (-0.80,0.83) {$\psi_A^\star$};
    \node[cB]   at ( 1.19,0.46) {$\psi_B^\star$};
    \node[cU]   at ( 0.16,0.98) {$\psi^\star$};
    \node[cA]   at (-0.64,1.15) {$\theta_A$};
    \node[cB]   at ( 1.36,1.04) {$\theta_B$};
    \node[cU]   at ( 0.45,1.06) {$\bar\theta_w$};
    \node[cU]   at ( 0.33,0.43) {$\psi^{\mathrm{sp}}$};

    \end{tikzpicture}%
    }
    \caption{\textbf{The impact-space view of social alignment.} Preferences ($\theta_n$, dashed) and the feasible impact set $\bar\Psi$ live in the same Euclidean space; for finitely many deployment queries, $\bar\Psi$ is a convex, origin-symmetric polytope. Preferences differ in magnitude and need not themselves be feasible impacts: both $\theta_A$ and $\theta_B$ lie outside $\bar\Psi$. Because welfare $U_n(\psi)=\theta_n^\top\psi$ is linear, each agent's ideal impact $\psi_n^\star=\arg\max_{\psi\in\bar\Psi}\theta_n^\top\psi$ is the vertex of $\bar\Psi$ farthest in the direction of $\theta_n$, which in general is \emph{not} collinear with $\theta_n$. The utilitarian optimum $\psi^\star$ maximizes welfare in the weighted-average direction $\bar\theta_w=\sum_n w_n\theta_n$, where the iso-welfare line $\{\psi:\bar\theta_w^\top\psi=\text{const}\}$ is tangent to $\bar\Psi$. The dashed segment $\operatorname{conv}(\psi_A^\star,\psi_B^\star)$ is the set of expected impacts attainable by random dictatorships (Section~\ref{sec:strategy}), with $\psi^{\mathrm{sp}}$ the equal-weight random dictatorship.}
    \label{fig:overall}
\end{figure}
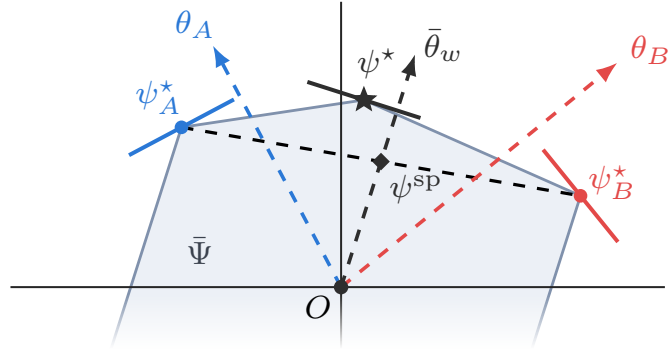

This lets us rewrite social alignment as a linear optimization over the set of feasible impacts:
%
\begin{equation}
    \label{eq:social-choice-problem}
    \psi^\star
    \in
    \argmax_{\psi \in \bar{\Psi}}
    \sum_{n=1}^N w_n \theta_n^\top \psi,
    \qquad
    \bar\Psi
    =
    \operatorname{closure}\{\psi(\theta):\theta\in\mathbb R^k\}
    \subseteq S_{\mathrm{dep}} .
\end{equation}
Here, $S_{\mathrm{dep}}$ is the subspace spanned by deployment queries, and $\bar\Psi$ is the feasible set of impacts achievable by some model parameter.\footnote{We optimize over the closure because deterministic limits can be approached to arbitrary precision.} The following theorem is our main contribution: it characterizes the geometry of the impact space and connects impacts to model parameters.

\begin{theorem}
\label{thm:welfare-geometry}
Let $\Psi=\{\psi(\theta):\theta\in\mathbb R^k\}$ and 
$\bar\Psi=\operatorname{cl}(\Psi)$, and let $S_{\mathrm{dep}}$ be the subspace of $\mathbb R^k$ spanned by the deployment queries. Then:
\begin{enumerate}
    \item $\bar\Psi\subseteq S_{\mathrm{dep}}$ is the origin-symmetric \emph{zonotope} $\bar\Psi=\{\mathbb E_{z\sim q_{\mathrm{dep}}}[t_z\,\alpha_z]:t_z\in[-\tfrac12,\tfrac12]\}$ generated by the deployment queries, and $\Psi=\operatorname{relint}\bar\Psi$, the interior of $\bar\Psi$ relative to $S_{\mathrm{dep}}$. In particular, $\bar\Psi$ is convex and centrally symmetric ($\bar\Psi=-\bar\Psi$), and $\psi$ is antisymmetric ($\psi(-\theta) = -\psi(\theta)$) with $\psi(0) = 0$.
    \item On $S_{\mathrm{dep}}$, the map $\psi$ is a $C^\infty$ diffeomorphism from $S_{\mathrm{dep}}$ onto $\Psi$. Every feasible impact $\psi\in\Psi$ is implemented by a unique parameter in $S_{\mathrm{dep}}$.
    \item $\psi$ is bounded, with 
    $
        \sup_{\theta\in\mathbb R^k}\|\psi(\theta)\|
        \leq
        \frac12\,\mathbb E_{z\sim q_{\mathrm{dep}}}\|\alpha_z\|.
    $
    \item Scaling is monotone in the direction of $\theta$:
    $
        \theta^\top \frac{\partial \psi(\nu\theta)}{\partial \nu}
        =
        \mathbb E_{z\sim q_{\mathrm{dep}}}
        \left[
            (\alpha_z^\top\theta)^2
            v(\nu\alpha_z^\top\theta)
        \right]
        \geq 0,
    $
    where $v(x)=\sigma'(x)$.
    \item The deterministic boundary in direction $\theta$ is
    $
        \lim_{\nu\to\infty}\psi(\nu\theta)
        =
        \frac12
        \mathbb E_{z\sim q_{\mathrm{dep}}}
        \left[
            \alpha_z\operatorname{sign}(\alpha_z^\top\theta)
        \right].
    $
    For generic $\theta$ (i.e., $\alpha_z^\top\theta\neq0$ for all $z\in\operatorname{supp}(q_{\mathrm{dep}})$), this limit is the unique maximizer of $\theta^\top\psi$ over $\bar\Psi$, and it is a vertex.
\end{enumerate}
\end{theorem}

This structure allows us to restate the utilitarian optimal model as the boundary point of $\bar{\Psi}$ in the direction of the weighted average preference $\bar{\theta}_w$.

\begin{proposition}
    \label{prop:asymptotic-alignment}
    The limit $\lim_{\nu\to\infty}\psi(\nu \, \bar{\theta}_w)$ exists and is a utilitarian-optimal impact: $\lim_{\nu\to\infty}\psi(\nu \, \bar{\theta}_w)\in\argmax_{\psi\in\bar\Psi}\bar\theta_w^\top\psi$. In particular, we may take $\psi^\star=\lim_{\nu\to\infty}\psi(\nu \, \bar{\theta}_w)$, and this maximizer is unique whenever $\alpha_z^\top\bar\theta_w\neq0$ for all $z\in\operatorname{supp}(q_{\mathrm{dep}})$.
\end{proposition}

Boundary impacts require parameter norms to grow without bound: Proposition~\ref{prop:finite-scale} in Appendix~\ref{app:scaling} shows that the relative welfare loss at scale $\nu$ is at most $2e^{-\nu m_w}$, where $m_w$ is the smallest nonzero margin $|\alpha_z^\top\bar\theta_w|$, so the scale needed for a given accuracy diverges near indifference and practical norm limits can prevent exact implementation.

Appendix \ref{app:scaling} gives the scaling formula and its implications for individual welfare. More generally, translating the alignment problem into impact space and considering the entire problem from the perspective of welfare brings \textit{externalities} into focus as the objects we would like to regulate as we design alignment protocols. In the remaining sections of this paper, we discuss how our framework helps us straightforwardly address these externalities. 




\section{Strategyproof Alignment in Impact Space}\label{sec:strategy}

When participants can influence alignment, they may misreport their preferences to improve their own welfare. Misreporting to distort alignment in a self-serving way imposes welfare externalities on others. In economics, the field of mechanism design regulates these by identifying procedures for combining preferences that are \emph{strategyproof}---i.e., such that no agent can benefit from misreporting.


The impact-space representation lets us apply classical mechanism design: we study option-set mechanisms~\citep{barbera_peleg_1990}, voting-by-issues rules, and random dictatorships. The latter can be implemented by one model parameter, exactly or in the limit, and arise from anonymous RLHF in the deterministic limit with query-independent labeling weights.






\begin{definition}
     A \textbf{mechanism} is a function $\mu : \mathbb R^{N \times k} \to \Delta(\bar{\Psi})$ that maps a report profile $\hat\theta=(\hat\theta_1,\dots,\hat\theta_N)$ to a probability distribution over the space of impact vectors $\bar{\Psi}$.
\end{definition}

Throughout this section we restrict attention to \emph{generic} profiles and reports---those with $\alpha_z^\top\theta_n\neq0$ for every $z\in\operatorname{supp}(q_{\mathrm{dep}})$ and every $n$---so that vote signs and ideal impacts $\psi_n^\star=\argmax_{\psi\in\bar\Psi}\theta_n^\top\psi$ are single-valued (Theorem~\ref{thm:welfare-geometry}).

Intuitively, we want a mechanism to satisfy two basic requirements. The first, as we have discussed, is \emph{strategyproofness}: an agent should not be able to obtain a better outcome, according to their true preferences, by misreporting those preferences. 

\begin{definition}
    A mechanism is \textbf{strategyproof} if, for all preference profiles $\theta \in \mathbb R^{N \times k}$, all $n \in [N]$, and all $\theta_n'$,
    \begin{equation*}
        \mathbb E\big[\theta_n^\top \mu(\theta)\big] \geq \mathbb E\big[\theta_n^\top \mu(\theta_n',\theta_{-n})\big],
    \end{equation*}
    where $(\theta_{n}',\theta_{-n})$ denotes the profile obtained by replacing agent $n$'s true preferences with $\theta_n'$. 
\end{definition}

The second is \emph{unanimity}: if all agents have exactly the same preferences, there is no social conflict to resolve, and the mechanism should simply choose the impact vector that is best according to that shared preference. 
%

\begin{definition}
    A mechanism $\mu$ is \textbf{unanimous} if, for any profile $\theta$ such that $\theta_n=\bar{\theta}$ for all $n \in [N]$, $\mu(\theta)$ is the point mass on $\argmax_{\psi \in \bar{\Psi}} \bar{\theta}^\top \psi$ (a singleton, since $\bar\theta$ is generic).
\end{definition}

We point out a powerful consequence of the linearity of $\theta_n^\top\psi$: that a stochastic mechanism can be fully characterized by its \emph{expected} impact vector at each profile.
\begin{lemma}\label{lem:mean-sufficiency}
    Let $\bar\mu(\theta)=\mathbb E_{\psi\sim\mu(\theta)}[\psi]\in\bar\Psi$ be the expected impact. Then $\mathbb E[\theta_n^\top\mu(\theta)]=\theta_n^\top\bar\mu(\theta)$, so strategyproofness and unanimity depend on $\mu$ only through the expected-impact map $\bar\mu:\mathbb R^{N\times k}\to\bar\Psi$.
\end{lemma}
This enables us to characterize classes of alignment protocols satisfying strategyproofness and unanimity while restricting attention to deterministic expected-impact maps.

\subsection{The General Family: Option-Set Mechanisms}

The impact space reformulation enables us to recognize social alignment as an instance of strategyproof social choice, which can be characterized generally using the option-set principle \citep{barbera_peleg_1990}. 
For a report profile, let $$O_n(\theta_{-n})=\{\bar\mu(\theta_n',\theta_{-n}):\theta_n'\in\mathbb R^k\}\subseteq\bar\Psi$$ be agent $n$'s \emph{option set}: the impacts it can induce by varying its own report, holding the others' fixed.

\begin{lemma}\label{thm:option-set}
    A mechanism $\mu$ is strategyproof if and only if there exist menus $M_n(\theta_{-n})\subseteq\bar\Psi$, each depending only on the other agents' reports, such that, at every profile, the mechanism deploys each agent's most-preferred impact from its menu:
    \begin{equation}
        \bar\mu(\theta)\in\argmax_{\psi\in M_n(\theta_{-n})}\theta_n^\top\psi\qquad\forall n.
    \end{equation}
    In particular, we can set $M_n(\theta_{-n})=O_n(\theta_{-n})$ to be the option set itself.
\end{lemma}

The lemma holds for any convex feasible set. It says that a strategyproof alignment protocol is equivalent to one that posts a menu of achievable model behaviors to each participant, which they cannot themselves enlarge, and deploys their favorite. 
The menu form is exact but implicit: an explicit characterization of all strategyproof mechanisms for expected-utility agents is a long-standing open problem even in the classical case of lotteries over finitely many alternatives \citep{barbera2011survey}. 
We therefore identify structured subfamilies of the menu family.

\subsection{Reduction 1: Voting by Deployment Query}

As established by Theorem~\ref{thm:welfare-geometry}, the impact set is a zonotope generated by the deployment queries. 
By decomposing the expectation over deployment queries into its individual terms, we can construct a strategyproof mechanism that aggregates the population's preferences query by query. 
To ensure strategyproofness, the mechanism must guarantee that agents cannot benefit by exaggerating the magnitude of their reports, which can be accomplished by reducing preference information to directional votes on each query. Write $s_{z,n}=\operatorname{sign}(\alpha_z^\top\theta_n)\in\{\pm1\}$ for how agent $n$ would decide query $z$.\footnote{For $(z,n)$ such that $\alpha_z^\top\theta_n = 0$, an arbitrary tie-breaking rule can be used, which does not affect welfare and therefore does not affect strategyproofness.}

\begin{definition}\label{def:voting-by-issues}
    A \textbf{voting-by-issues} mechanism is specified by, for each query $z$, an aggregator $f_z:\{\pm1\}^N\to[-1,1]$ that is nondecreasing in each argument with $f_z(1,\dots,1)=1$ and $f_z(-1,\dots,-1)=-1$. Its expected impact is
    \begin{equation}
        \bar\mu(\theta)=\tfrac12\,\mathbb E_{z\sim q_{\mathrm{dep}}}\big[f_z(s_{z,1},\dots,s_{z,N})\,\alpha_z\big].
    \end{equation}
\end{definition}

\begin{theorem}\label{thm:voting-by-issues}
    Every voting-by-issues mechanism is strategyproof and unanimous. 
\end{theorem}

Voting-by-issues rules are the impact-space analogue of \emph{generalized median voter schemes} \citep{barbera1993generalized}: each deployment query is settled by a monotone aggregation of the participants' preferred sides. 
Our setting differs from those schemes in that mechanisms are stochastic, outcomes range over the impact space, and preferences are linear in the same space. The fact that the impact set is a zonotope generated by the queries, and therefore exhibits central symmetry, makes this per-query aggregation feasible. 

\subsection{Reduction 2: Random Dictatorships}

\begin{definition}\label{def:random-dictatorship}
    A mechanism is a \textbf{random dictatorship} with weights $\lambda \in \Delta_N$ if $\mu(\theta) = \sum_{n=1}^N \lambda_n \delta(\psi_n^\star)$. By Lemma~\ref{lem:mean-sufficiency}, this is impact-equivalent to $\bar\mu(\theta)=\sum_{n=1}^N\lambda_n\psi_n^\star$.
\end{definition}

\begin{proposition}\label{prop:random-dictatorship}
A random dictatorship is impact-equivalent to the voting-by-issues mechanism whose aggregators are the vote averages $f_z(s_{z,\cdot})=\sum_n\lambda_n s_{z,n}$. It is therefore strategyproof and unanimous.
\end{proposition}

Classical results force strategyproofness to coincide with random dictatorship alone, but only on richer preference domains: a universal domain over the outcomes \citep{gibbard1977}, or strictly convex single-peaked preferences over a convex set \citep{dutta_peters_sen_2002}. Linearity is a smaller domain and so admits a richer strategyproof class, with random dictatorship as a focal example.

As defined, a random dictatorship draws a single individual, with probability $\lambda_n$, and deploys their ideal impact. By Lemma~\ref{lem:mean-sufficiency} it is impact-equivalent to the more natural protocol of choosing an individual at random on a query-by-query basis and deciding each query the way that individual would. A naive implementation of either would fit a separate optimal model for each individual (e.g., a collection of person-specific LoRA adapters) and randomly select among them at deployment. The weights $\lambda$ are fixed by the designer before reports arrive, not estimated from them; since Proposition~\ref{prop:random-dictatorship} holds for every $\lambda\in\Delta_N$, truthful reporting is weakly dominant even for participants who do not know $\lambda$, provided $\lambda$ does not depend on their reports.

This lottery need not be implemented literally: it can be compiled into a \emph{single} static parameter, exactly when its target impact lies in the relative interior of $\bar\Psi$ and as a limit otherwise.

\begin{theorem}\label{thm:single-sp}
    The random dictatorship with weights $\lambda\in\Delta_N$ has target impact $\bar\psi_\lambda=\sum_{n=1}^N\lambda_n\psi_n^\star\in\bar\Psi$. This impact is attained by a single model parameter, which then reproduces the random dictatorship's welfare for every agent, if and only if $\bar\psi_\lambda\in\Psi$, in which case the parameter is $\theta^{\mathrm{sp}}_\lambda=\psi^{-1}(\bar\psi_\lambda)$. 
    When $\bar\psi_\lambda$ lies on the boundary of $\bar\Psi$, it is instead the limit of the impacts of a sequence of single parameters.
\end{theorem}

Even though $\theta^{\mathrm{sp}}_\lambda$ is one fixed parameterization, its stochastic responses reproduce, in expectation, the same impact as randomizing over agents' ideal models. 
Mathematically, it follows from the invertibility of $\psi$; conceptually, it reflects that model stochasticity is itself a form of compromise between conflicting individuals (see Figure~\ref{fig:constrainedwelfare}).





Having established that a random dictatorship is strategyproof and can be compiled into a single model parameter, we now ask how much doing so constrains the achievable welfare. First, note that we can select a set of weights $\lambda$ to achieve any $\psi \in \Psi^{\mathrm{sp}} = \operatorname{conv}(\{\psi_n^\star\}_{n=1}^N)$. 
Intuitively, each $\psi_n^\star$ is the support point of $\bar\Psi$ in the direction of $\theta_n$. If agent preferences are spread over all directions, then their convex hull $\Psi^{\mathrm{sp}}$ will approximate $\bar\Psi$. Welfare will be low when social welfare maximization selects a ``compromise'' impact vector that is dissimilar to all individual vectors. 

\begin{proposition}
    \label{prop:unbounded-sp-distortion}
    Let $\bar{\Psi}\subset\mathbb R^k$ be any compact, convex, origin-symmetric set
    with $\dim(\bar{\Psi}) \geq 2$, suppose the welfare weights $w\in\Delta_N$ are not a point mass, and write $U^\star(\theta)=\max_{\psi\in\bar{\Psi}}\bar{\theta}_w^\top\psi$ for the optimal utilitarian welfare. For any weights $\lambda \in \Delta_N$ and any $\delta>0$, there exists a preference profile $\theta$ with $U^\star(\theta)>0$ at which the random dictatorship $\mu_\lambda^{\textrm{sp}}$ attains at most a $\delta$ fraction of the optimum,
    $\mathbb{E}[U(\mu_\lambda^{\textrm{sp}}(\theta))] \leq \delta\, U^\star(\theta)$. Hence, for non-degenerate welfare weights, the worst-case distortion $U^\star(\theta)/\mathbb{E}[U(\mu_\lambda^{\textrm{sp}}(\theta))]$ of every random dictatorship is unbounded.\footnote{Some restriction on $w$ is necessary: if $w=\delta_i$, then $U^\star(\theta)=\max_{\psi\in\bar\Psi}\theta_i^\top\psi=\theta_i^\top\psi_i^\star$, while origin symmetry gives $\theta_i^\top\psi\geq-U^\star(\theta)$ for all $\psi\in\bar\Psi$; hence every profile satisfies $\mathbb E[U(\mu_\lambda^{\mathrm{sp}}(\theta))]=\sum_n\lambda_n\theta_i^\top\psi_n^\star\geq(2\lambda_i-1)\,U^\star(\theta)$, so for $\lambda_i>\tfrac12$ the distortion is bounded.}
\end{proposition}

Intuitively, the proof shows that this worst case arises when an individual or small group has a very strong preference along some dimension on which the overwhelming majority has a weak preference in the opposite direction. In such cases, a random dictatorship will implement the weak preferences of the majority, which causes outsized harm to the minority.\footnote{In the applied domains that we study, this structure does not generally arise and the strategyproof mechanism's welfare sub-optimality is bounded (see Figure \ref{fig:expectedwelfare} in Appendix \ref{app:empirics}).} Proposition~\ref{prop:sp-cost} in Appendix~\ref{app:sp-cost} bounds the welfare loss of any random dictatorship by the dispersion of preferences around the utilitarian direction, and Table~\ref{tab:sp-cost} shows the loss on our four datasets falling as the population's average majority margin rises.

These results also shed new light on the standard practice of RLHF. Anonymous RLHF fits a single parameter $\hat{\theta}^{\mathrm{RLHF}}$ to the pooled labels, ignoring who produced them. Our unifying result is that its deployed impact is exactly that of a simple stochastic protocol---answer each query by sampling a labeler and following their choice---and therefore depends on the training data only through the population's per-query label frequencies.

\begin{proposition}\label{prop:rlhf-impact}
    Suppose $q_{\mathrm{train}}$ and $q_{\mathrm{dep}}$ share the same query marginal, and let $\hat{\theta}^{\mathrm{RLHF}}$ be any stationary point of the anonymous RLHF objective. 
    Writing $\mu_z=\mathbb E_{n\sim q_{\mathrm{train}}(\cdot\mid z)}[\mathbb E[c\mid n,z]]$ for the population's label frequency on query $z$, $$\psi(\hat{\theta}^{\mathrm{RLHF}})=\mathbb E_{z\sim q_{\mathrm{dep}}}\big[(\mu_z-\tfrac12)\,\alpha_z\big].$$ 
    That is, the deployed model is impact-equivalent to the random-labeler protocol that answers each deployment query $z$ by sampling an agent $n\sim q_{\mathrm{train}}(\cdot\mid z)$ and a response $c\sim p(\cdot\mid n,z)$.
\end{proposition}

The identity holds because the RLHF first-order condition matches the first moment of $\sigma(\alpha_z^\top\hat{\theta}^{\mathrm{RLHF}})$ against $\alpha_z$ to that of $\mu_z$, and the impact $\psi$ is exactly that moment. In the deterministic limit of labeler accuracy, this turns into a per-query vote count, and the impact into that of a voting-by-issues rule.

\begin{corollary}\label{cor:rlhf-convergence}
    Suppose, as in Proposition~\ref{prop:rlhf-impact}, that $q_{\mathrm{train}}$ and $q_{\mathrm{dep}}$ share the same query marginal. 
    In the deterministic labeling limit $\mathbb E[c\mid n,z]\to\mathbf 1\{\alpha_z^\top\theta_n>0\}$, anonymous RLHF converges to the voting-by-issues mechanism with per-query aggregators $f_z(s_{z,\cdot})=\sum_n q_{\mathrm{train}}(n\mid z)\,s_{z,n}$, and is therefore strategyproof and unanimous. 
    When the labeling weights are query-independent, $q_{\mathrm{train}}(n\mid z)=\lambda_n$, it is the random dictatorship with weights $\lambda$.
\end{corollary}

This connects RLHF to the strategyproof voting-by-issues family, with random dictatorship when labeling weights are query-independent. Corollary~\ref{cor:finite-eta} bounds the gap at finite labeling precision. 

\section{Linear Constraints Can Enforce Externality Bounds and Trade-offs}
\label{sec:constrained-welfare}



In practice, alignment designers have objectives beyond strictly maximizing utilitarian welfare or ensuring strategyproofness. A designer may want to pursue high total welfare while guaranteeing that no individual or group is harmed beyond some threshold or, more generally, to guarantee normative commitments about how the benefits and costs of alignment are distributed.

The impact-space formulation makes this task straightforward. Because welfare is linear in the impact vector, any welfare guarantee expressible as a linear inequality becomes a halfspace constraint on $\psi$. As we show below, a floor on individual welfare, a bound on the ratio of personal harm to social benefit, or a constraint that no individual imposes externalities above some level can all be expressed in this way. 

When augmented with such a linear constraint, the designer's problem becomes:
\begin{equation}
    \psi^{\mathrm{wel}}
    \in
    \argmax_{\psi\in\bar{\Psi}}
    \bar{\theta}_w^\top\psi
    \quad
    \mathrm{s.t.}
    \quad
    \gamma_n^\top\psi\geq c_n
    \ \forall\, n,
\end{equation}
for constraint directions $\gamma_n \in \mathbb{R}^k$ and thresholds $c_n \in \mathbb{R}$. Since each constraint is a halfspace and $\bar{\Psi}$ is convex, closed, and bounded, this is a linear optimization with linear constraints whenever the feasible set is nonempty. Different choices of $\gamma_n$ and $c_n$ encode different commitments; we develop welfare floors below and two further families, public spirit and participation-externality bounds, in Appendix~\ref{app:constrained-welfare}.

\begin{theorem}
\label{thm:constrained-linear-social-choice}
If $\bar{\Psi}\cap\bigcap_n\{\psi:\gamma_n^\top\psi\geq c_n\}$ is nonempty, then an optimum $\psi^{\mathrm{wel}}$ exists. Moreover, there are multipliers $\eta_n\geq0$ such that
$$
    \psi^{\mathrm{wel}}
    \in
    \argmax_{\psi\in\bar{\Psi}}
    \left(
        \bar{\theta}_w+\sum_{n=1}^N \eta_n\gamma_n
    \right)^\top\psi,
$$
with complementary slackness $\eta_n(\gamma_n^\top\psi^{\mathrm{wel}}-c_n)=0$ for all $n$. Equivalently, only the binding constraints $\mathcal B=\{n:\eta_n>0\}$ tilt the objective, giving effective direction $\bar{\theta}_w+\sum_{n\in\mathcal B}\eta_n\gamma_n$. The optimal utilitarian value is weakly decreasing and concave in the thresholds $c$.
\end{theorem}

We now go through one example: bounding harm. A natural welfare guarantee is to require every agent's welfare to exceed a floor $b_n$. This is the special case $\gamma_n=\theta_n$ and $c_n=b_n$.

\begin{definition}[Harm-bounded alignment]
For welfare floors $b\in\mathbb R^N$,
$$
    \psi_b^{\mathrm h}
    \in
    \argmax_{\psi\in\bar{\Psi}}
    \bar{\theta}_w^\top\psi
    \quad
    \mathrm{s.t.}
    \quad
    \theta_n^\top\psi\geq b_n
    \ \forall n .
$$
\end{definition}

\begin{corollary}[of Theorem~\ref{thm:constrained-linear-social-choice}]
\label{cor:bounded-harm}
The harm-bounded optimum satisfies
$$
    \psi_b^{\mathrm h}
    \in
    \argmax_{\psi\in\bar{\Psi}}
    \left(
        \sum_{n=1}^N (w_n+\eta_n)\theta_n
    \right)^\top\psi,
$$
where $\eta_n\geq0$ and $\eta_n(\theta_n^\top\psi_b^{\mathrm h}-b_n)=0$ for all $n$.
\end{corollary}

The corollary has a simple interpretation. Instead of maximizing with fixed weights $w_n$, the planner maximizes with effective weights $w_n+\eta_n$. The extra term $\eta_n$ is positive exactly for agents who would fall below their required welfare level without additional protection. Thus the constrained solution can be implemented as an ordinary weighted-welfare maximization problem, but with extra weight placed on the agents that need protection.

Useful choices of welfare floor include: (i) $b_n=0$, so no agent is worse off than under a random model; (ii) $b_n=\theta_n^\top\psi_w^{\mathrm{sp}}$, so every agent weakly prefers the mechanism to the welfare-weighted strategyproof solution; (iii) $b_n=-\epsilon h_{\bar{\Psi}}(\theta_n)$, so agent $n$ can lose at most an $\epsilon$ fraction of their maximum achievable welfare; and (iv) $b_n=b^\star=\max_{\psi\in\bar\Psi}\min_m\theta_m^\top\psi$, the max-min value, which confines the mechanism to the max-min-optimal impacts and selects among them one that is Pareto efficient for strictly positive $w$, since any Pareto improvement would keep every agent above $b^\star$ and raise the objective.\footnote{The same construction applies to groups by replacing $\theta_n$ with a group preference $\bar{\theta}_{\mathcal G}=\sum_{n\in\mathcal G}w_n\theta_n$ and $b_n$ with a group floor.}

\begin{figure}[t]
    \centering
    \includegraphics[width=\linewidth]{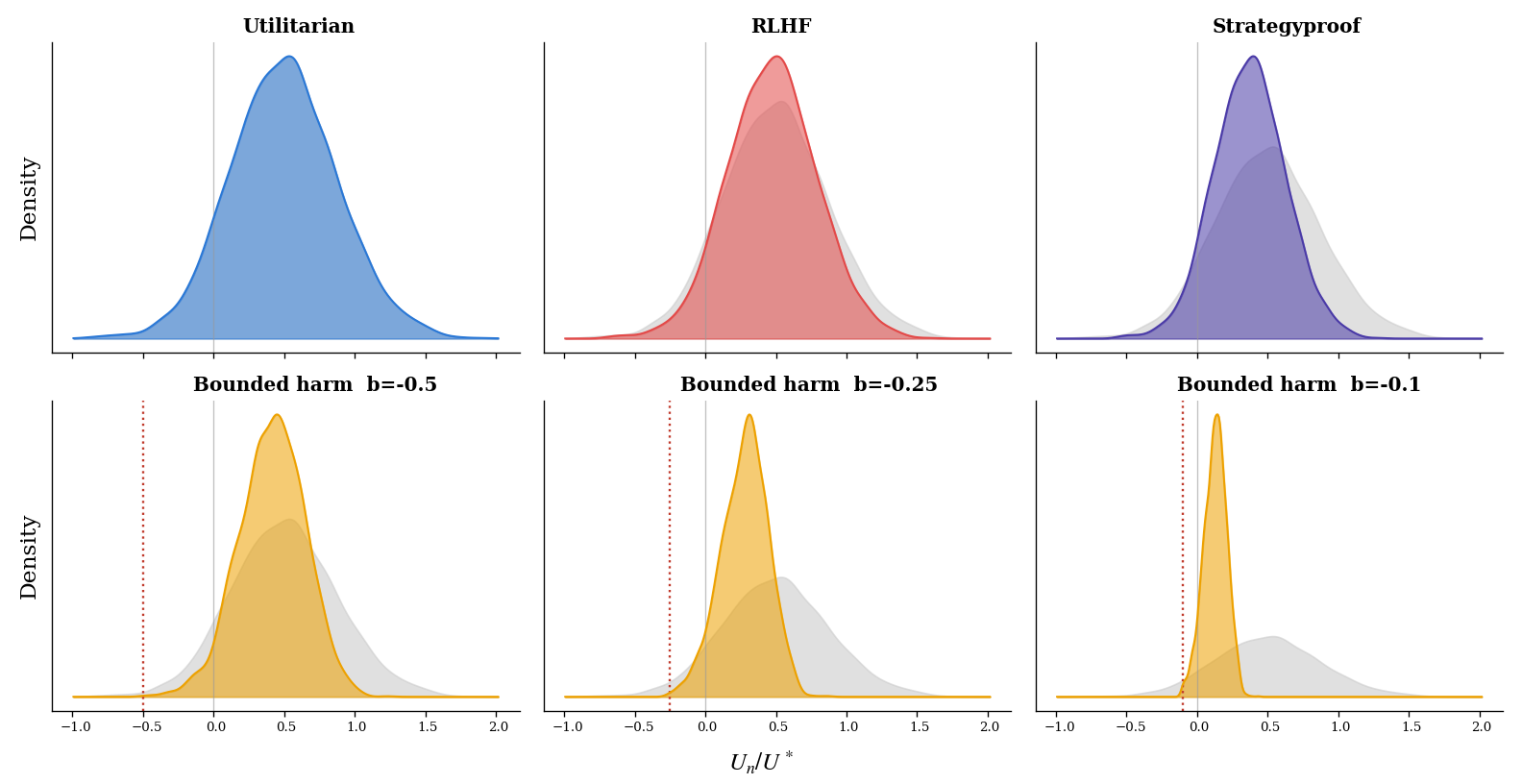}
    \caption{Community Alignment dataset: distribution of individual welfare $U_n/U^\star$ across the 2387 annotators under each mechanism, at high choice precision ($\beta=100$). Top: utilitarian, RLHF, and strategyproof. Bottom: harm-bounded alignment at three welfare floors $b$; the faint gray curve repeats the utilitarian distribution for reference, and the dotted red line marks the floor. Tightening the floor (toward $0$) binds more agents and compresses the welfare distribution, removing the harmed left tail.}
    \label{fig:welfaredist}
\end{figure}


\section{Empirical Illustrations}
\label{sec:empirical}

\paragraph{Setup.} We illustrate the framework on real human pairwise-choice data from four domains: kidney allocation \citep{keswani2024pros}, charitable food distribution \citep{lee2019webuildai}, trolley problems from Moral Machine \citep{awad2018moral}, and preferences over LLM responses from the Community Alignment dataset \citep{zhang2025cultivating}. In each domain, options are described by interpretable features ($k$ ranging from $5$ to $20$) and we fit each participant's linear preference $\theta_n$ with a Bradley--Terry--Luce model on their own choices; a precision parameter $\beta$ scales all preferences, with choices becoming deterministic as $\beta\to\infty$. For Community Alignment, responses are first embedded by a frozen sentence encoder (Qwen3-Embedding-8B) and projected onto a $K=15$-dimensional interpretable basis using the method of \citet{wojtowicz2026weights}, yielding $N=2387$ annotators. Taking the observed queries as the deployment distribution, the impact zonotope $\bar\Psi$---and with it every mechanism in this paper---can be computed exactly. Dataset details and preprocessing are given in Appendix~\ref{app:empirics}.

\paragraph{Mechanisms and measures.} We compare the deployed impacts of the utilitarian optimum (Proposition~\ref{prop:asymptotic-alignment}), anonymous RLHF (Proposition~\ref{prop:rlhf-impact}), the uniform-weight random dictatorship compiled into a single model parameter (Theorem~\ref{thm:single-sp}), and harm-bounded alignment at several welfare floors (Corollary~\ref{cor:bounded-harm}). For each mechanism we compute every agent's welfare $U_n$ relative to the utilitarian optimum $U^\star$, as well as the person-on-person participation externalities $\Gamma_{n,m}$ (Appendix~\ref{app:fullalignment}).


\paragraph{Findings.} On Community Alignment, welfare varies widely across agents under the utilitarian, RLHF, and strategyproof mechanisms, and under all three some agents are actively harmed (Figure~\ref{fig:welfaredist}). Under harm-bounded alignment no agent's welfare falls below the chosen floor, at a cost in mean welfare that grows as the floor tightens. The externality measures show that this cost comes from a small number of agents whose harm constraint binds, and that including these agents is costly to everyone else. 
Across the four datasets, agents disagree about the \emph{direction} of their preferences in Moral Machine and Community Alignment, but mostly about the \emph{magnitude} of their preferences in the kidney and food-rescue domains (Figure~\ref{fig:preferencedist} in Appendix~\ref{app:empirics}). As a result, the strategyproof mechanism loses only a limited fraction of the optimal welfare in these datasets (Figure~\ref{fig:expectedwelfare}), because the adversarial structure behind Proposition~\ref{prop:unbounded-sp-distortion} does not arise in them. At high choice precision, welfare also varies least across agents under this mechanism (Figures~\ref{fig:welfarevar} and~\ref{fig:peragentwelfare}). Under the utilitarian, RLHF, and strategyproof mechanisms, welfare varies more across agents as the choice precision $\beta$ grows and individual choices become more deterministic (Figure~\ref{fig:constrainedwelfare}).

\section{Limitations and Discussion}

\paragraph{Limitations.} Our framework builds on, and therefore relies on, the assumption that preferences are linear in a fixed feature map $\phi$, and the quality of that feature map limits what our guarantees mean. If the representation of the model being aligned is held fixed and only a linear head is trained on top of it, our guarantees apply exactly to the induced model class. If alignment training also changes the representation, or if the chosen features omit aspects of people's preferences, then the guarantees apply only to that approximation. Measuring this representation error is an important empirical step that we leave to future work. The $15$-dimensional interpretable basis of our Community Alignment demonstration is a convenience for exposition, not a requirement of the method, and we do not claim that fifteen dimensions suffice in practice. Additionally, our assumption that the query distribution is known, or at least can be approximated, may not hold in practice. 

\paragraph{Discussion.} We have shown that linear social choice provides the right setting in which to reformulate alignment as a linear optimization problem. By analyzing welfare not as a second-order consequence of model parameterization but as a first-order object, we are able to develop an impact-space framework that directly informs how we design mechanisms that satisfy various desiderata, such as strategyproofness. The designer's degrees of freedom this framework exposes---the weights $\lambda$, the floors $b_n$, and the externality tolerances---concentrate power in whoever sets them, but such choices are not new: anonymous RLHF already fixes $\lambda$ implicitly through who enters the labeling pool, and it already distributes harm across people. Impact space makes the welfare consequences of these choices explicit and contestable, but who should make them is a governance question that we do not address. The information granularity resulting from the impact space analysis leads to a rich set of future questions: can we use this framework to constrain the welfare impacts of one demographic group on another? Can every alignment mechanism (traditional RLHF, DPO, Max-Min, Nash) be characterized by the patterns of externalities that it creates over the population of individuals, and do we have the power to regulate these in a unified way?

    %
        
        





\section*{Acknowledgments}
This work was partially supported by the National Science Foundation under grant IIS-2229881; by the Office of Naval Research under grants N00014-24-1-2704 and N00014-25-1-2153; and by grants from the Cooperative AI Foundation and the Foresight Institute. Michelle Si was supported by an NSF GRFP.

\bibliography{references}

\appendix

\section{Proofs for Sections \ref{sec:prelim} and \ref{sec:framework}}

Throughout the proofs, let
$
    \Psi=\{\psi(\theta):\theta\in\mathbb R^k\}
$
and $\bar\Psi=\operatorname{cl}(\Psi)$. 
We optimize over $\bar\Psi$, but $\psi^{-1}$ is used only on $\Psi$, where the impact map is invertible on the deployment subspace.

For any non-empty closed convex set $A$, the support function of $A$ is 
\begin{equation}
    h_A(x) = \sup_{a \in A} x^\top a.
\end{equation} 
%


\begin{proof}[Proof of Theorem~\ref{thm:welfare-geometry}]

    \hspace{0em}
    \begin{enumerate}
        \item Write the impact map in terms of its per-query coordinates
        \begin{equation}
            \psi(\theta)
            =\mathbb E_{z\sim q_{\mathrm{dep}}}\big[t_z(\theta)\,\alpha_z\big],
            \qquad
            t_z(\theta)=\sigma(\alpha_z^\top\theta)-\tfrac12\in(-\tfrac12,\tfrac12),
        \end{equation}
        and let $\mathcal Z=\{\mathbb E_{z\sim q_{\mathrm{dep}}}[t_z\alpha_z]:t_z\in[-\tfrac12,\tfrac12]\}$ be the origin-symmetric zonotope generated by $\{\tfrac12 q_{\mathrm{dep}}(z)\alpha_z\}$, i.e.\ the linear image $M([-\tfrac12,\tfrac12]^{Z})$ of the cube (one coordinate per query $z\in Z$) under $M(t)=\mathbb E_{z\sim q_{\mathrm{dep}}}[t_z\alpha_z]$. 
        Since each $\alpha_z\in S_{\mathrm{dep}}$, both $\Psi$ and $\mathcal Z$ lie in $S_{\mathrm{dep}}$. 
        Because a linear map carries the relative interior of a convex set onto the relative interior of its image, $M$ maps the open cube onto $\operatorname{relint}\mathcal Z$, the interior of $\mathcal Z$ relative to $S_{\mathrm{dep}}$. 
        Note that this is nonempty, since the generators of $\mathcal Z$ span $S_{\mathrm{dep}}$. Moreover, because $t(\theta)$ lies in the open cube, $\Psi\subseteq\operatorname{relint}\mathcal Z$. 
        A zonotope is convex and origin-symmetric. We have $\psi(0)=0$ since $\sigma(0)=\tfrac12$, and $\psi(-\theta)=-\psi(\theta)$ since $\sigma(-x)-\tfrac12=-(\sigma(x)-\tfrac12)$.
        
        \item The map $\psi$ is the gradient of the convex log-partition potential
        \begin{equation}
            \Phi(\theta)=\mathbb E_{z\sim q_{\mathrm{dep}}}\big[\log(1+e^{\alpha_z^\top\theta})-\tfrac12\alpha_z^\top\theta\big],
            \qquad
            \nabla\Phi=\psi,\quad
            \nabla^2\Phi(\theta)=\mathbb E_{z\sim q_{\mathrm{dep}}}\big[v(\alpha_z^\top\theta)\,\alpha_z\alpha_z^\top\big],
        \end{equation}
        with $v=\sigma'>0$. The Hessian is positive definite on $S_{\mathrm{dep}}$, i.e., if $u\in S_{\mathrm{dep}}$ satisfies $u^\top\nabla^2\Phi(\theta)\,u=\mathbb E_{z\sim q_{\mathrm{dep}}}[v(\alpha_z^\top\theta)(\alpha_z^\top u)^2]=0$, then $\alpha_z^\top u=0$ for every $z\in\operatorname{supp}(q_{\mathrm{dep}})$, so $u\perp S_{\mathrm{dep}}$, which together with $u\in S_{\mathrm{dep}}$ forces $u=0$. Hence $\Phi$ restricted to $S_{\mathrm{dep}}$ is finite, $C^\infty$, and strictly convex. Given that it is finite and differentiable everywhere, it is essentially smooth, and is therefore a convex function of Legendre type \citep[\S26]{rockafellar1970convex}.
        We restrict attention to $S_{\mathrm{dep}}$ throughout, as on the ambient space $\Phi$ has flat directions and $\mathcal Z$ has empty interior whenever $S_{\mathrm{dep}}\neq\mathbb R^k$, so all interiors below are relative to $S_{\mathrm{dep}}$. The recession function of $\Phi$ is
        \begin{equation}
            \lim_{\nu\to\infty}\frac{\Phi(\nu u)}{\nu}
            =\mathbb E_{z\sim q_{\mathrm{dep}}}\big[(\alpha_z^\top u)_+-\tfrac12\,\alpha_z^\top u\big]
            =\tfrac12\,\mathbb E_{z\sim q_{\mathrm{dep}}}\big[\,|\alpha_z^\top u|\,\big]
            =h_{\mathcal Z}(u),
        \end{equation}
        the support function of $\mathcal Z$. By \citet[Theorem~13.3]{rockafellar1970convex}, the support function of $\operatorname{dom}\Phi^*$ is the recession function of $\Phi$; hence $\operatorname{cl}(\operatorname{dom}\Phi^*)=\mathcal Z$, and so $\operatorname{relint}(\operatorname{dom}\Phi^*)=\operatorname{relint}\mathcal Z$. 
        By \citet[Theorem~26.5]{rockafellar1970convex}, the gradient map of a Legendre-type function is a bijection from the interior of its domain onto the interior of the domain of its convex conjugate, with continuous inverse $\nabla\Phi^*$; here these interiors are $S_{\mathrm{dep}}$ and $\operatorname{relint}\mathcal Z$, so $\psi=\nabla\Phi$ is a bijection from $S_{\mathrm{dep}}$ onto $\operatorname{relint}\mathcal Z$, and, by the inverse function theorem, a $C^\infty$ diffeomorphism. 
        Therefore $\Psi=\operatorname{relint}\mathcal Z$, so $\bar\Psi=\operatorname{cl}(\operatorname{relint}\mathcal Z)=\mathcal Z$ is the stated zonotope.

        \item From $|t_z(\theta)|\le\tfrac12$, $\|\psi(\theta)\|\le\mathbb E[\|\alpha_z\|\,|t_z(\theta)|]\le\tfrac12\mathbb E[\|\alpha_z\|]$.
        \item Differentiating under the expectation, $\frac{\partial}{\partial\nu}\psi(\nu\theta)=\mathbb E[(\alpha_z^\top\theta)v(\nu\alpha_z^\top\theta)\alpha_z]$, so
        \begin{equation}
            \theta^\top\frac{\partial}{\partial\nu}\psi(\nu\theta)
            =\mathbb E_{z\sim q_{\mathrm{dep}}}\big[(\alpha_z^\top\theta)^2 v(\nu\alpha_z^\top\theta)\big]\ge0.
        \end{equation}
        \item By dominated convergence, $\lim_{\nu\to\infty}t_z(\nu\theta)=\tfrac12\operatorname{sign}(\alpha_z^\top\theta)$, so
        \begin{equation}
            \lim_{\nu\to\infty}\psi(\nu\theta)
            =\tfrac12\,\mathbb E_{z\sim q_{\mathrm{dep}}}\big[\alpha_z\operatorname{sign}(\alpha_z^\top\theta)\big],
        \end{equation}
        which attains $\max_{\psi\in\bar\Psi}\theta^\top\psi=\tfrac12\,\mathbb E_{z\sim q_{\mathrm{dep}}}|\alpha_z^\top\theta|$. For generic $\theta$ (i.e., $\alpha_z^\top\theta\neq0$ for all $z\in\operatorname{supp}(q_{\mathrm{dep}})$), it is the unique maximizer: any maximizer may be written $\mathbb E_{z\sim q_{\mathrm{dep}}}[t_z\,\alpha_z]$ with $t_z\in[-\tfrac12,\tfrac12]$, and $\theta^\top\psi=\mathbb E_{z\sim q_{\mathrm{dep}}}[t_z\,(\alpha_z^\top\theta)]$ attains this value only if $t_z=\tfrac12\operatorname{sign}(\alpha_z^\top\theta)$ for every $z\in\operatorname{supp}(q_{\mathrm{dep}})$, which determines the point. As the singleton exposed face of the polytope $\bar\Psi$ in direction $\theta$, it is a vertex; when $\theta=\theta_n$, it is agent $n$'s ideal impact $\psi_n^\star$.
    \end{enumerate}
\end{proof}

\begin{proof}[Proof of Proposition~\ref{prop:asymptotic-alignment}]
    By Theorem~\ref{thm:welfare-geometry}(5), the limit exists and equals $p=\tfrac12\,\mathbb E_{z\sim q_{\mathrm{dep}}}[\alpha_z\operatorname{sign}(\alpha_z^\top\bar\theta_w)]$, and $p\in\bar\Psi$ because $\bar\Psi$ is closed and $p$ is a limit of points of $\Psi$. Moreover,
    $
        \bar\theta_w^\top p=\tfrac12\,\mathbb E_{z\sim q_{\mathrm{dep}}}[\,|\alpha_z^\top\bar\theta_w|\,]=h_{\bar\Psi}(\bar\theta_w),
    $
    the support function of the zonotope computed in the proof of Theorem~\ref{thm:welfare-geometry}, so $p$ attains the maximum of $\bar\theta_w^\top\psi$ over $\bar\Psi$, i.e.\ $p\in\argmax_{\psi\in\bar\Psi}\bar\theta_w^\top\psi$. For uniqueness under genericity: any maximizer can be written $\psi=\mathbb E_{z\sim q_{\mathrm{dep}}}[t_z\alpha_z]$ with $t_z\in[-\tfrac12,\tfrac12]$, and $\bar\theta_w^\top\psi=\mathbb E[t_z(\alpha_z^\top\bar\theta_w)]$ attains $\tfrac12\mathbb E|\alpha_z^\top\bar\theta_w|$ only if $t_z=\tfrac12\operatorname{sign}(\alpha_z^\top\bar\theta_w)$ for every $z$ with $\alpha_z^\top\bar\theta_w\neq0$; when $\alpha_z^\top\bar\theta_w\neq0$ for all $z\in\operatorname{supp}(q_{\mathrm{dep}})$, this pins down every representation, so the maximizer equals $p$.
\end{proof}

\subsection{Scaling and Individual Welfare}\label{app:scaling}
For agent $n$,
\[
    \frac{\partial U_n(\nu\theta)}{\partial \nu}
    =
    \theta_n^\top \frac{\partial \psi(\nu\theta)}{\partial \nu}
    =
    \mathbb E_{z\sim q_{\mathrm{dep}}}
    \left[
        (\alpha_z^\top\theta_n)
        (\alpha_z^\top\theta)
        v(\nu\alpha_z^\top\theta)
    \right].
\]
Thus scaling \textit{helps} agents whose preferences align with the model direction across deployment queries, and \textit{harms} agents whose preferences covary negatively.

\begin{proposition}\label{prop:finite-scale}
Let $u_z=\alpha_z^\top\bar\theta_w$, let $U^\star=\sup_{\theta}U(\theta)=\tfrac12\,\mathbb E_{z\sim q_{\mathrm{dep}}}\big[|u_z|\big]$, and let $m_w=\min\{|u_z|:z\in\operatorname{supp}(q_{\mathrm{dep}}),\,u_z\neq0\}$. For every $\nu>0$,
\begin{equation}
    U^\star-U(\nu\bar\theta_w)
    =\mathbb E_{z\sim q_{\mathrm{dep}}}\big[|u_z|\,\sigma(-\nu|u_z|)\big]
    \le 2e^{-\nu m_w}\,U^\star .
\end{equation}
Consequently, under a norm constraint $\|\theta\|\le r$, the parameter $(r/\|\bar\theta_w\|)\,\bar\theta_w$ has relative welfare loss at most $\epsilon$ whenever $r\ge(\|\bar\theta_w\|/m_w)\log(2/\epsilon)$.
\end{proposition}

\begin{proof}
    Since $|\sigma(t)-\tfrac12|\le\tfrac12$, $U(\theta)=\mathbb E_z[u_z(\sigma(\alpha_z^\top\theta)-\tfrac12)]\le\tfrac12\mathbb E_z\big[|u_z|\big]$ for every $\theta$, and Proposition~\ref{prop:asymptotic-alignment} shows the bound is attained in the limit, so $U^\star=\tfrac12\mathbb E_z\big[|u_z|\big]$. For $u_z>0$, $\tfrac12 u_z-u_z(\sigma(\nu u_z)-\tfrac12)=u_z\,\sigma(-\nu u_z)$; for $u_z<0$ the same algebra gives $|u_z|\,\sigma(-\nu|u_z|)$; for $u_z=0$ the term vanishes. Taking expectations gives the identity. Since $\sigma(-x)\le e^{-x}$ for $x\ge0$ and $|u_z|\ge m_w$ whenever $u_z\neq0$, the gap is at most $e^{-\nu m_w}\mathbb E_z\big[|u_z|\big]=2e^{-\nu m_w}U^\star$. The last claim follows by setting $\nu=r/\|\bar\theta_w\|$ and solving $2e^{-\nu m_w}\le\epsilon$.
\end{proof}

\section{Proofs for Section \ref{sec:strategy}}

\begin{proof}[Proof of Lemma~\ref{lem:mean-sufficiency}]
    By linearity of $\theta_n^\top(\cdot)$, $\mathbb E[\theta_n^\top\mu(\theta)]=\theta_n^\top\mathbb E_{\psi\sim\mu(\theta)}[\psi]=\theta_n^\top\bar\mu(\theta)$. Strategyproofness is an inequality between such expectations and so depends only on $\bar\mu$. 
    For unanimity, on the generic domain where it is defined, the point $\psi^\star(\bar\theta)=\argmax_{\psi\in\bar\Psi}\bar\theta^\top\psi$ is an exposed vertex---and hence an extreme point---of $\bar\Psi$ (Theorem~\ref{thm:welfare-geometry}). A distribution whose mean is an extreme point is the point mass, so $\mu(\theta,\dots,\theta)=\delta(\psi^\star(\bar\theta))$ iff $\bar\mu(\theta,\dots,\theta)=\psi^\star(\bar\theta)$.
\end{proof}

\begin{proof}[Proof of Lemma~\ref{thm:option-set}]
    By Lemma~\ref{lem:mean-sufficiency} we restrict attention to $\bar\mu$. The argument follows that of \cite{barbera_peleg_1990}.
    \begin{enumerate}
        \item[$\Leftarrow$] Suppose such menus $M_n$ exist. 
        Fix $n$, $\theta_{-n}$, and true type $\theta_n$. 
        For any misreport $\theta_n'$, the mechanism applied at the profile $(\theta_n',\theta_{-n})$ gives $\bar\mu(\theta_n',\theta_{-n})\in M_n(\theta_{-n})$, and applied at $\theta$ it says $\bar\mu(\theta)$ maximizes $\theta_n^\top\psi$ over $M_n(\theta_{-n})$. Hence $\theta_n^\top\bar\mu(\theta)\ge\theta_n^\top\bar\mu(\theta_n',\theta_{-n})$. 
        Thus, it is strategyproof.

        \item[$\Rightarrow$] Suppose $\mu$ is strategyproof and take $M_n(\theta_{-n}) = O_n(\theta_{-n})=\{\bar\mu(\theta_n',\theta_{-n}):\theta_n'\in\mathbb R^k\}$, which depends only on $\theta_{-n}$. 
        Then $\bar\mu(\theta)\in O_n(\theta_{-n})$ trivially, and strategyproofness says $\theta_n^\top\bar\mu(\theta)\ge\theta_n^\top\psi$ for every $\psi\in O_n(\theta_{-n})$, i.e. $\bar\mu(\theta)\in\argmax_{\psi\in O_n(\theta_{-n})}\theta_n^\top\psi$.
    \end{enumerate}

\end{proof}

    %

\begin{proof}[Proof of Theorem~\ref{thm:voting-by-issues}]
    Fix agent $n$ and the others' reports. By Definition~\ref{def:voting-by-issues}, agent $n$'s expected utility is
    \begin{equation}
        \theta_n^\top\bar\mu(\theta)=\tfrac{1}{2}\,\mathbb E_{z\sim q_{\mathrm{dep}}}\big[f_z(s_{z})\,(\alpha_z^\top\theta_n)\big],
    \end{equation}
    a sum over queries in which $n$ controls only its own votes $\{s_{z,n}\}$. 
    For each $z$, this is increasing in $f_z$ when $\alpha_z^\top\theta_n>0$ and decreasing when $\alpha_z^\top\theta_n<0$. 
    Since $f_z$ is nondecreasing in $s_{z,n}$, it is maximized by the truthful vote $s_{z,n}=\operatorname{sign}(\alpha_z^\top\theta_n)$. 
    The truthful report realizes all these votes simultaneously, so it maximizes every term and hence the sum.  
    
    For unanimity, note that identical reports give $s_{z,n} = s_z = \operatorname{sign}(\alpha_z^\top\bar\theta)$, so $f_z=s_z$ and $\bar\mu=\tfrac12\mathbb E_z[s_z\alpha_z]=\psi^\star(\bar\theta)$, which is unanimous by Lemma~\ref{lem:mean-sufficiency}.
\end{proof}

\begin{proof}[Proof of Proposition~\ref{prop:random-dictatorship}]
    The aggregator $f_z(s_{z})=\sum_n\lambda_n s_{z,n}$ is nondecreasing in each vote (given $\lambda_n\ge0$) with $f_z(\pm\mathbf 1)=\pm1$, so it defines a voting-by-issues mechanism, which is strategyproof and unanimous by Theorem~\ref{thm:voting-by-issues}. 
    Its expected impact is
    $
        \tfrac12\mathbb E_z\big[\big(\textstyle\sum_n\lambda_n s_{z,n}\big)\alpha_z\big]
        =\sum_n\lambda_n\cdot\tfrac12\mathbb E_z[s_{z,n}\alpha_z]
        =\sum_n\lambda_n\psi_n^\star,
    $
    the random dictatorship with weights $\lambda$.
\end{proof}

\begin{proof}[Proof of Theorem~\ref{thm:single-sp}]
    Each ideal impact $\psi_n^\star$ lies in $\bar\Psi$, which is convex (Theorem~\ref{thm:welfare-geometry}), so the target impact $\bar\psi_\lambda=\sum_n\lambda_n\psi_n^\star\in\bar\Psi$. 
    By Theorem~\ref{thm:welfare-geometry}(1), the impacts attainable by a single parameter are exactly $\Psi=\{\psi(\theta):\theta\in\mathbb R^k\}$, the relative interior of $\bar\Psi$, and by Theorem~\ref{thm:welfare-geometry}(2), $\psi$ is a bijection from $S_{\mathrm{dep}}$ onto $\Psi$. 
    Hence $\bar\psi_\lambda$ is attained by a single parameter if and only if $\bar\psi_\lambda\in\Psi$, in which case the unique such parameter in $S_{\mathrm{dep}}$ is $\theta^{\mathrm{sp}}_\lambda=\psi^{-1}(\bar\psi_\lambda)$, with impact exactly $\bar\psi_\lambda$.
    This is the mean impact of the random dictatorship (Proposition~\ref{prop:random-dictatorship}). 
    Since welfare depends on a mechanism only through its expected impact (Lemma~\ref{lem:mean-sufficiency}), this single parameter reproduces the random dictatorship's welfare for every agent.

    When $\bar\psi_\lambda$ instead lies on the boundary of $\bar\Psi$, it remains a limit of single-parameter impacts: since $0=\psi(0)\in\Psi=\operatorname{relint}\bar\Psi$ and $\bar\psi_\lambda\in\bar\Psi$, the half-open segment $\big\{(1-\varepsilon)\bar\psi_\lambda: \varepsilon\in(0,1]\big\}$ lies in $\Psi$. 
    Thus $\theta^{\mathrm{sp}}_{\lambda,\varepsilon}:=\psi^{-1}\big((1-\varepsilon)\bar\psi_\lambda\big)$ is well-defined for every $\varepsilon\in(0,1)$, and $\psi\big(\theta^{\mathrm{sp}}_{\lambda,\varepsilon}\big)=(1-\varepsilon)\bar\psi_\lambda\to\bar\psi_\lambda$ as $\varepsilon\to0$.
\end{proof}

\begin{proof}[Proof of Proposition~\ref{prop:unbounded-sp-distortion}]
    \textit{Choice of directions.} Since $\bar\Psi$ is convex and origin-symmetric, its affine hull is the linear span $S=\operatorname{span}(\bar\Psi)$ and $0\in\operatorname{relint}\bar\Psi$; hence $h_{\bar\Psi}(d)>0$ for every nonzero $d\in S$. The support function $h_{\bar\Psi}$ is finite and convex on $S$, hence differentiable at Lebesgue-almost every direction in $S$, and since $\bar\Psi\subseteq S$ it is differentiable at $d\in S$ exactly when $\argmax_{\psi\in\bar\Psi}d^\top\psi$ is a singleton. Fix such a $d_1\in S$ with $d_1\neq0$ and let $v$ denote the unique maximizer in direction $d_1$; then $d_1^\top v=h_{\bar\Psi}(d_1)>0$ and $v\neq0$, and by symmetry $-v$ is the unique maximizer in direction $-d_1$. Because $\dim(\bar{\Psi})\geq2$ and $v\in S$, we may choose $d_2\in S$ with $d_2\perp v$ and $d_2\neq0$, and then $h_{\bar{\Psi}}(d_2)>0$.

    \textit{Profile.} Since $w$ is not a point mass, fix $i$ with $0<w_i<1$. For $\varepsilon > 0$, set
    \begin{equation}
        \theta_n =
        \begin{cases}
            d_1 + \varepsilon\,d_2 & n = i \\[4pt]
            -\dfrac{w_i}{1-w_i}\,d_1 + \varepsilon\,d_2 & n \neq i.
        \end{cases}
    \end{equation}
    (The coefficient for $n\neq i$ is common to all agents, so zero welfare weights are permitted.) The $d_1$ components cancel in the welfare-weighted average:
    \begin{equation}
        \bar{\theta}_w
        = \sum_n w_n\,\theta_n
        = w_i\,d_1-\frac{w_i}{1-w_i}\,(1-w_i)\,d_1+\varepsilon\,d_2
        = \varepsilon\,d_2,
    \end{equation}
    so the first-best welfare is $U^\star = \varepsilon\, h_{\bar{\Psi}}(d_2)>0$.

    \textit{Welfare of the random dictatorship.} As $\varepsilon\downarrow0$, the normalized direction of $\theta_i$ converges to that of $d_1$ and the normalized direction of every $\theta_n$, $n \neq i$, converges to that of $-d_1$. The correspondence $d\mapsto\argmax_{\psi\in\bar\Psi}d^\top\psi$ is upper hemicontinuous (Berge's maximum theorem, using compactness of $\bar\Psi$) and singleton-valued at $\pm d_1$, so \emph{every} selection of maximizers satisfies $\psi_i^\star\to v$ and $\psi_n^\star\to-v$ for $n\neq i$; no continuity or strict-convexity assumption on $\bar\Psi$ is needed, and the argument covers the zonotopes of Theorem~\ref{thm:welfare-geometry}. Therefore
    \begin{equation}
        \psi^{\mathrm{sp}}
        = \sum_n \lambda_n\,\psi_n^\star
        \;\longrightarrow\;
        \lambda_i\,v-(1-\lambda_i)\,v=(2\lambda_i - 1)\,v
        \qquad(\varepsilon\downarrow0),
    \end{equation}
    and, since $d_2^\top v=0$ by construction,
    \begin{equation}
        U^{\mathrm{sp}}
        =\bar{\theta}_w^\top\psi^{\mathrm{sp}}
        = \varepsilon\,d_2^\top\psi^{\mathrm{sp}}
        = \varepsilon\cdot o(1).
    \end{equation}
    Combining with $U^\star=\varepsilon\,h_{\bar{\Psi}}(d_2)>0$, the welfare ratio is $U^{\mathrm{sp}}/U^\star=o(1)$ as $\varepsilon\downarrow0$, which is below $\delta$ once $\varepsilon$ is small enough. The construction places no restriction on $\lambda$.
\end{proof}

\subsection{The Cost of Strategyproofness}
\label{app:sp-cost}

Proposition~\ref{prop:unbounded-sp-distortion} shows that the welfare loss of a random dictatorship can be arbitrarily large in the worst case. The following bound shows that the loss is controlled by the dispersion of preferences around the utilitarian direction, so strategyproofness is cheap when preferences are concentrated and expensive when they are dispersed.

\begin{proposition}\label{prop:sp-cost}
For any $\lambda\in\Delta_N$, writing $U^\star=\bar\theta_w^\top\psi^\star$,
\[
    U^\star-\mathbb E\big[U(\mu_\lambda^{\mathrm{sp}}(\theta))\big]
    \le\operatorname{diam}(\bar\Psi)\sum_{n=1}^N\lambda_n\|\theta_n-\bar\theta_w\|
    \le\mathbb E_{z\sim q_{\mathrm{dep}}}\big[\|\alpha_z\|\big]\sum_{n=1}^N\lambda_n\|\theta_n-\bar\theta_w\| .
\]
\end{proposition}

\begin{proof}
By Lemma~\ref{lem:mean-sufficiency}, $\mathbb E[U(\mu_\lambda^{\mathrm{sp}}(\theta))]=\sum_n\lambda_n\bar\theta_w^\top\psi_n^\star$. Since $\psi_n^\star$ maximizes $\theta_n^\top\psi$ over $\bar\Psi$, $\theta_n^\top(\psi^\star-\psi_n^\star)\le0$, hence
\[
    \bar\theta_w^\top(\psi^\star-\psi_n^\star)
    =(\bar\theta_w-\theta_n)^\top(\psi^\star-\psi_n^\star)+\theta_n^\top(\psi^\star-\psi_n^\star)
    \le\|\bar\theta_w-\theta_n\|\operatorname{diam}(\bar\Psi).
\]
Averaging over $\lambda$ gives the first inequality; the second uses $\operatorname{diam}(\bar\Psi)\le2\sup_{\psi\in\bar\Psi}\|\psi\|\le\mathbb E_{z\sim q_{\mathrm{dep}}}\big[\|\alpha_z\|\big]$ (Theorem~\ref{thm:welfare-geometry}(3)).
\end{proof}

To measure agreement over deployment choices empirically, we use the average majority margin $M=\mathbb E_{z\sim q_{\mathrm{dep}}}\big|\tfrac1N\sum_n\operatorname{sign}(\alpha_z^\top\theta_n)\big|$, the expected gap between the shares of agents favoring the majority and minority options, which is $1$ under unanimity and $0$ under an even split. Table~\ref{tab:sp-cost} in Appendix~\ref{app:empirics} reports the welfare loss of the uniform random dictatorship against $M$ on our four datasets; the loss falls as $M$ rises.

\section{Anonymous RLHF Through the Impact Identity}
\label{app:rlhf}

This section collects our results on anonymous RLHF. We first prove the identity and derive its mechanism consequences: the deterministic labeling limit (Corollary~\ref{cor:rlhf-convergence}) and manipulability at finite labeling precision (Corollary~\ref{cor:finite-temp-manipulable}). We then develop the marginal welfare effect of reweighting training sources and the resulting contrast between RLHF stationarity and utilitarian efficiency (Theorem~\ref{thm:stationarity-vs-efficiency}). Finally, we characterize the identification limits imposed by the identity's information bottleneck: which impacts reweighting can and cannot recover (Section~\ref{app:relabel}).

\subsection{The Identity and Its Mechanism Consequences}

\begin{proof}[Proof of Proposition~\ref{prop:rlhf-impact}]
    The anonymous RLHF objective is the cross-entropy loss
    \begin{equation}
        -\,\mathbb E_{(n,z)\sim q_{\mathrm{train}}}\,\mathbb E_{c\mid n,z}\!\left[c\log\sigma(\alpha_z^\top\theta)+(1-c)\log\bigl(1-\sigma(\alpha_z^\top\theta)\bigr)\right],
    \end{equation}
    whose stationarity condition is
    \begin{equation}
        \mathbb E_{z\sim q_{\mathrm{train}}}\!\left[\bigl(\mu_z-\sigma(\alpha_z^\top\hat{\theta}^{\mathrm{RLHF}})\bigr)\,\alpha_z\right]=0,
        \qquad
        \mu_z=\mathbb E_{n\sim q_{\mathrm{train}}(\cdot\mid z)}\bigl[\mathbb E[c\mid n,z]\bigr].
    \end{equation}
    Since $q_{\mathrm{train}}$ and $q_{\mathrm{dep}}$ share the same query marginal, the same moment identity $\mathbb E[\sigma(\alpha_z^\top\hat{\theta}^{\mathrm{RLHF}})\alpha_z]=\mathbb E[\mu_z\alpha_z]$ holds under $q_{\mathrm{dep}}$. Subtracting $\tfrac12\mathbb E_{z\sim q_{\mathrm{dep}}}[\alpha_z]$ from both moments,
    \begin{equation}
        \psi(\hat{\theta}^{\mathrm{RLHF}})
        =\mathbb E_{z\sim q_{\mathrm{dep}}}\!\left[\bigl(\sigma(\alpha_z^\top\hat{\theta}^{\mathrm{RLHF}})-\tfrac12\bigr)\alpha_z\right]
        =\mathbb E_{z\sim q_{\mathrm{dep}}}\!\left[\bigl(\mu_z-\tfrac12\bigr)\alpha_z\right].
    \end{equation}
    The argument uses only first-order stationarity and the shared query marginal. For the equivalence claim, the random-labeler protocol answers query $z$ with choice probability $\mu_z$, so it delivers welfare $\theta_n^\top\mathbb E_{z\sim q_{\mathrm{dep}}}[(\mu_z-\tfrac12)\alpha_z]$ to each agent $n$---by the display above, the same as the deployed model. Note, finally, that the right-hand side of the identity is a function of $\{\mu_z\}$ alone, so the deployed impact depends on the training data only through the per-query label frequencies.
\end{proof}

\begin{proof}[Proof of Corollary~\ref{cor:rlhf-convergence}]
    In the deterministic labeling limit $\mathbb E[c\mid n,z]\to\mathbf 1\{\alpha_z^\top\theta_n>0\}$, so $\mu_z\to\sum_n q_{\mathrm{train}}(n\mid z)\,\mathbf 1\{\alpha_z^\top\theta_n>0\}$. Using $\mathbf 1\{x>0\}-\tfrac12=\tfrac12\operatorname{sign}(x)$ and $\sum_n q_{\mathrm{train}}(n\mid z)=1$,
    \begin{equation}
        \mu_z-\tfrac12=\tfrac12\sum_n q_{\mathrm{train}}(n\mid z)\,\operatorname{sign}(\alpha_z^\top\theta_n)=\tfrac12\, f_z(s_{z,\cdot}),
        \qquad
        f_z(s_{z,\cdot})=\sum_n q_{\mathrm{train}}(n\mid z)\,s_{z,n}.
    \end{equation}
    By Proposition~\ref{prop:rlhf-impact} and dominated convergence ($|(\mu_z-\tfrac12)\alpha_z|\le\tfrac12\|\alpha_z\|$), the deployed impact converges to $\tfrac12\mathbb E_{z\sim q_{\mathrm{dep}}}[f_z(s_{z,\cdot})\alpha_z]$, the impact of the voting-by-issues mechanism of Definition~\ref{def:voting-by-issues} with aggregator $f_z$. Each $f_z$ is a convex combination of the votes, hence nondecreasing in every argument with $f_z(\pm\mathbf 1)=\pm1$, so the mechanism is strategyproof and unanimous by Theorem~\ref{thm:voting-by-issues}. When $q_{\mathrm{train}}(n\mid z)=\lambda_n$ for all $z$, $f_z(s_{z,\cdot})=\sum_n\lambda_n s_{z,n}$ and the impact is $\sum_n\lambda_n\psi_n^\star$, the random dictatorship with weights $\lambda$ (Proposition~\ref{prop:random-dictatorship}).
\end{proof}

At finite labeling precision, the identity instead exposes a manipulation channel: the frequencies $\{\mu_z\}$, and hence the deployed impact, respond continuously to the \emph{intensity} of an agent's reported preferences.

\begin{corollary}[Finite-temperature RLHF is manipulable]
\label{cor:finite-temp-manipulable}
Suppose, as in Corollary~\ref{cor:rlhf-convergence}, that $q_{\mathrm{train}}$ and $q_{\mathrm{dep}}$ share the same query marginal, that the labeling weights are query-independent, $q_{\mathrm{train}}(n\mid z)=\lambda_n$, and that labels follow the Bradley--Terry--Luce model at the \emph{reported} preferences, $\mathbb E[c\mid n,z]=\sigma(\alpha_z^\top\hat\theta_n)$. Then the anonymous RLHF impact is the $\lambda$-average of the reported individual impacts,
$$
    \psi\big(\hat\theta^{\mathrm{RLHF}}\big)=\sum_{n=1}^N\lambda_n\,\psi(\hat\theta_n),
$$
and the induced mechanism is not strategyproof: for every agent $n$ with $\lambda_n>0$ and generic true preference $\theta_n$, the utility of reporting $\nu\theta_n$ is strictly increasing in $\nu$, so any $\nu>1$ strictly improves on truthful reporting and no optimal report exists.
\end{corollary}

\begin{proof}
Since every $\mu_z=\sum_n\lambda_n\sigma(\alpha_z^\top\hat\theta_n)$ lies in $(0,1)$, the cross-entropy objective is strictly convex and coercive on $S_{\mathrm{train}}=S_{\mathrm{dep}}$ (the supports of the query marginals coincide), so a unique stationary point exists there. By Proposition~\ref{prop:rlhf-impact} and $\sum_n\lambda_n=1$,
$$
    \psi\big(\hat\theta^{\mathrm{RLHF}}\big)
    =\mathbb E_{z\sim q_{\mathrm{dep}}}\big[(\mu_z-\tfrac12)\alpha_z\big]
    =\sum_{n=1}^N\lambda_n\,\mathbb E_{z\sim q_{\mathrm{dep}}}\big[\big(\sigma(\alpha_z^\top\hat\theta_n)-\tfrac12\big)\alpha_z\big]
    =\sum_{n=1}^N\lambda_n\,\psi(\hat\theta_n).
$$
Fix agent $n$ and the others' reports. By the identity above, agent $n$'s deployment welfare from reporting $\theta'$ is $\lambda_n\,\theta_n^\top\psi(\theta')$ plus a term that does not depend on its report, and by Theorem~\ref{thm:welfare-geometry}(4),
$$
    \frac{d}{d\nu}\,\theta_n^\top\psi(\nu\theta_n)
    =\mathbb E_{z\sim q_{\mathrm{dep}}}\big[(\alpha_z^\top\theta_n)^2\,v(\nu\,\alpha_z^\top\theta_n)\big]>0
$$
for generic $\theta_n$. Hence reporting $\nu\theta_n$ with $\nu>1$ strictly improves on the truthful report. The supremum $h_{\bar\Psi}(\theta_n)$ of $\theta_n^\top\psi(\theta')$ over reports is approached along the ray $\nu\theta_n$ as $\nu\to\infty$ but is not attained by any report: for generic $\theta_n$ the unique maximizer $\psi_n^\star$ is an exposed vertex of $\bar\Psi$ and hence lies outside $\Psi=\operatorname{relint}\bar\Psi$.
\end{proof}

\cite{kleinebuening2025strategyproof} show that RLHF with stochastic labelers is not strategyproof: a labeler can manipulate its choice probabilities to steer the learned policy. Corollary~\ref{cor:rlhf-convergence} is complementary: in the deterministic limit, where labelers can only manipulate their choice directions (not intensities), RLHF converges to a strategyproof voting-by-issues rule. The additional manipulation channel available at finite temperature is exactly the nonlinearity that makes the finite-temperature mechanism deviate from this limit.

The deviation is small, however, whenever labelers are decisive on every deployment query.

\begin{corollary}\label{cor:finite-eta}
In the setting of Corollary~\ref{cor:finite-temp-manipulable}, suppose labels are truthful at precision $\eta>0$, $\mathbb E[c\mid n,z]=\sigma(\eta\,\alpha_z^\top\theta_n)$, and the profile is generic. Let $A=\mathbb E_{z\sim q_{\mathrm{dep}}}\big[\|\alpha_z\|\big]$ and $m_n=\min\{|\alpha_z^\top\theta_n|:z\in\operatorname{supp}(q_{\mathrm{dep}})\}$. Then
\[
    \Big\|\psi\big(\hat\theta^{\mathrm{RLHF}}_\eta\big)-\sum_{n=1}^N\lambda_n\psi_n^\star\Big\|
    \le A\sum_{n=1}^N\lambda_n e^{-\eta m_n}.
\]
If the cross-entropy objective is augmented with an $\ell_2$ penalty $\tfrac{\rho}{2}\|\theta\|^2$, the deployed impact is shifted by exactly $-\rho\,\hat\theta^{\mathrm{RLHF}}$, a vector of norm $\rho\|\hat\theta^{\mathrm{RLHF}}\|$.
\end{corollary}
\begin{proof}
By Corollary~\ref{cor:finite-temp-manipulable} with $\hat\theta_n=\eta\theta_n$, $\psi(\hat\theta^{\mathrm{RLHF}}_\eta)=\sum_n\lambda_n\psi(\eta\theta_n)$. By Theorem~\ref{thm:welfare-geometry}(5), $\psi_n^\star=\tfrac12\mathbb E_z[\alpha_z\operatorname{sign}(\alpha_z^\top\theta_n)]$, so with $a_z=\alpha_z^\top\theta_n$,
$\psi(\eta\theta_n)-\psi_n^\star=\mathbb E_z\big[\alpha_z\big(\sigma(\eta a_z)-\tfrac12-\tfrac12\operatorname{sign}(a_z)\big)\big]$. The bracket has magnitude $\sigma(-\eta|a_z|)\le e^{-\eta|a_z|}\le e^{-\eta m_n}$, giving $\|\psi(\eta\theta_n)-\psi_n^\star\|\le Ae^{-\eta m_n}$; averaging over $\lambda$ gives the bound. With the penalty, the stationarity condition in the proof of Proposition~\ref{prop:rlhf-impact} becomes $\mathbb E_z[(\mu_z-\sigma(\alpha_z^\top\hat\theta^{\mathrm{RLHF}}))\alpha_z]=\rho\,\hat\theta^{\mathrm{RLHF}}$, so $\psi(\hat\theta^{\mathrm{RLHF}})=\mathbb E_z[(\mu_z-\tfrac12)\alpha_z]-\rho\,\hat\theta^{\mathrm{RLHF}}$.
\end{proof}

\subsection{The Differential Companion: Marginal Alignment Externalities}
The identity pins down the deployed impact at a fixed training distribution; its differential companion describes how the impact---and hence welfare---moves as the training distribution is reweighted.
For an RLHF procedure that trains on a query distribution $q_{\mathrm{train}} \in \Delta([N] \times Z)$, we define the \textit{marginal} alignment externality: the first-order welfare effect of infinitesimally upweighting a specific training source $(n,z)$. This is the alignment analogue of the influence function from robust statistics \citep[see][for a review]{hammoudeh2024training}. Here $\hat\theta_{q}$ denotes the anonymous RLHF parameter trained on distribution $q$; recall that $U_m(\theta)=\theta_m^\top\psi(\theta)$ is agent $m$'s deployment welfare.

\begin{definition}[Marginal Alignment Externality]
    \label{def:marginal-externality}
    For a training source $(n,z)$, let $q_\varepsilon = (1-\varepsilon)\,q_{\mathrm{train}} + \varepsilon\,\delta_{(n,z)}$, where $\delta_{(n,z)}$ is a point mass on $(n,z)$. The \textbf{marginal alignment externality} of source $(n,z)$ on agent $m$ is
    \begin{equation}
        \gamma^{m}_{n,z}
        =
        \frac{d}{d\varepsilon}
        U_m(\hat{\theta}_{q_\varepsilon})
        \bigg|_{\varepsilon=0}
        =
        \theta_m^\top \frac{d\,\psi(\hat{\theta}_{q_\varepsilon})}{d\varepsilon}\bigg|_{\varepsilon=0},
    \end{equation}
    the inner product of $m$'s preference with the \textbf{marginal impact shift} from upweighting source $(n,z)$.
\end{definition}

The alignment externality $\Gamma$ (Appendix~\ref{app:fullalignment}) and the marginal externality $\gamma$ are related: the former is the discrete welfare change from participation, the latter the infinitesimal welfare change from reweighting. Both are inner products of preferences with impact shifts---$\Gamma$ uses the finite shift $\Delta_n = \psi - \psi_{-n}$, while $\gamma$ uses the infinitesimal shift $d\psi/d\varepsilon$---but are conceptually distinct: $\Gamma$ asks ``what if agent $n$ had never participated?'' while $\gamma$ asks ``what if we gave source $(n,z)$ slightly more weight?''

\subsection{Anonymous RLHF and the Closed-Form Marginal Externality}

Anonymous RLHF fits a single parameter $\hat\theta^{\mathrm{RLHF}}$ to pooled data, ignoring agent identities at training time:
\begin{equation}
    \hat{\theta}_{q_{\mathrm{train}}}^{\mathrm{RLHF}}
    \in
    \arg\max_{\theta\in\mathbb R^k}\ 
    -\mathbb E_{(n,z)\sim q_{\mathrm{train}}}\,
    \mathbb E_{c\mid n,z}\left[
        -c\log\sigma(\alpha_z^\top\theta)
    -(1-c)\log\big(1-\sigma(\alpha_z^\top\theta)\big)
    \right].
\end{equation}
%
The marginal externality has a closed form, obtained from the influence function of the estimator. Throughout the remainder of this section we assume the Bradley--Terry--Luce labeler model $\mathbb E[c\mid n,z]=\sigma(\alpha_z^\top\theta_n)$ and that the training queries span $\mathbb R^k$, so that $H$ below is positive definite and $\hat\theta_q$ is the unique stationary point of the training objective, depending smoothly on $q$ by the implicit function theorem. (Otherwise, every statement should be read as restricted to $S_{\mathrm{train}}=\operatorname{span}\{\alpha_z:z\in\operatorname{supp}(q_{\mathrm{train}})\}$, with $S_{\mathrm{dep}}\subseteq S_{\mathrm{train}}$ required.) The first-order condition is $g(\hat\theta^{\mathrm{RLHF}};q_{\mathrm{train}})=0$ with $g(\theta;q)=\mathbb E_{(n',z')\sim q}[(\sigma(\alpha_{z'}^\top\theta_{n'})-\sigma(\alpha_{z'}^\top\theta))\,\alpha_{z'}]$. Differentiating $g(\hat\theta_{q_\varepsilon};q_\varepsilon)=0$ at $\varepsilon=0$ and using the base condition $g(\hat\theta^{\mathrm{RLHF}};q_{\mathrm{train}})=0$ to cancel the distributional term gives the influence function
\begin{equation}\label{eq:rlhf-influence}
    \frac{d\,\hat\theta_{q_\varepsilon}}{d\varepsilon}\bigg|_{0}
    =\big(\sigma(\alpha_z^\top\theta_n)-\sigma(\alpha_z^\top\hat\theta^{\mathrm{RLHF}})\big)\,H^{-1}\alpha_z,
    \qquad
    H=\mathbb E_{(n',z')\sim q_{\mathrm{train}}}\big[v(\alpha_{z'}^\top\hat\theta^{\mathrm{RLHF}})\,\alpha_{z'}\alpha_{z'}^\top\big].
\end{equation}
Applying the Jacobian $\nabla\psi(\hat\theta^{\mathrm{RLHF}})=\mathbb E_{z'\sim q_{\mathrm{dep}}}[v(\alpha_{z'}^\top\hat\theta^{\mathrm{RLHF}})\alpha_{z'}\alpha_{z'}^\top]$ (Theorem~\ref{thm:welfare-geometry}) and summing over agents with the welfare weights $w$ yields a closed form for the \emph{social welfare effect} $\gamma^W_{n,z}=\sum_{m}w_m\gamma^m_{n,z}$ of source $(n,z)$:
\begin{equation}
    \label{eq:rlhf-externality}
    \gamma_{n,z}^{W}
    =
    \underbrace{\big(\sigma(\alpha_z^\top\theta_n) - \sigma(\alpha_z^\top\hat{\theta}^{\mathrm{RLHF}})\big)}_{\text{preference deviation on query } z}\,
    \underbrace{\big(\alpha_z^\top \eta\big)}_{\text{welfare sensitivity of query } z}
\end{equation}
where
\begin{equation}
    \eta
    =
    \mathbb E_{(n',z')\sim q_{\mathrm{train}}}\left[
        v(\alpha_{z'}^\top\hat{\theta}^{\mathrm{RLHF}})\,
        \alpha_{z'}\,\alpha_{z'}^\top
    \right]^{-1}\,
    \mathbb E_{z\sim q_{\mathrm{dep}}}\left[
        (\alpha_z^\top\bar{\theta}_w)\,v(\alpha_z^\top\hat{\theta}^{\mathrm{RLHF}})\,\alpha_z
    \right]
\end{equation}
is independent of $(n,z)$ and acts as a sufficient statistic for the welfare implications of parameter changes. Anonymous RLHF distorts welfare along directions that are simultaneously (i)~systematically misrepresented in the training data and (ii)~socially important at deployment.

\subsection{RLHF Stationarity vs.\ Utilitarian Efficiency}

The central result of this section is that RLHF's first-order condition and utilitarian efficiency impose qualitatively different requirements on the marginal externalities. RLHF ensures only that the \textit{social welfare effect} of marginal perturbations cancels on average across training sources; utilitarian efficiency requires that it vanishes for \textit{every} training source individually.

Equivalently, $\gamma^W_{n,z}=\bar{\theta}_w^\top \frac{d\,\psi(\hat{\theta}_{q_\varepsilon})}{d\varepsilon}\big|_{\varepsilon=0}$, the inner product of the utilitarian preference with the marginal impact shift from upweighting source $(n,z)$.

\begin{theorem}\label{thm:stationarity-vs-efficiency}
    
    \hspace{1ex}
    \begin{enumerate}
        \item \textbf{RLHF stationarity (average zero).} At the anonymous RLHF solution, the training-weighted average social welfare effect vanishes:
        \begin{equation}
            \mathbb E_{(n,z)\sim q_{\mathrm{train}}}\left[\gamma_{n,z}^W\right] = 0.
        \end{equation}
        Individual sources may have $\gamma_{n,z}^W>0$ (upweighting would improve welfare) or $\gamma_{n,z}^W<0$, but these cancel on average under $q_{\mathrm{train}}$.

        \item \textbf{Utilitarian efficiency (pointwise zero).} If the training weighting is \emph{welfare-optimal}---if $q_{\mathrm{train}}$ maximizes the deployed utilitarian welfare $\bar{\theta}_w^\top\psi(\hat\theta_q)$ over all reweightings $q$ of its sources---then the social welfare effect vanishes for every source individually:
        \begin{equation}
            \gamma_{n,z}^W = 0 \qquad \forall\,(n,z) \in \mathrm{supp}(q_{\mathrm{train}}).
        \end{equation}
    \end{enumerate}
\end{theorem}

Anonymous RLHF guarantees only the average condition. Whenever it leaves some source with $\gamma_{n,z}^W\neq0$, upweighting the sources with positive marginal effect (and downweighting those with negative effect) strictly improves deployed welfare to first order.

\begin{proof}[Proof of Theorem~\ref{thm:stationarity-vs-efficiency}]
    \emph{(1).} By the closed form \eqref{eq:rlhf-externality},
    $
        \mathbb E_{(n,z)\sim q_{\mathrm{train}}}[\gamma^W_{n,z}]
        =\big(\mathbb E_{(n,z)\sim q_{\mathrm{train}}}[(\sigma(\alpha_z^\top\theta_n)-\sigma(\alpha_z^\top\hat\theta^{\mathrm{RLHF}}))\,\alpha_z]\big)^\top\eta
        =0,
    $
    because the RLHF first-order condition makes the bracketed vector vanish.

    \emph{(2).} Let $W(q)=U(\hat\theta_q)=\bar\theta_w^\top\psi(\hat\theta_q)$ be the deployed welfare as a function of the training weighting. Since $q_\varepsilon=(1-\varepsilon)q_{\mathrm{train}}+\varepsilon\delta_{(n,z)}$, the definition of $\gamma^W_{n,z}$ is the directional derivative toward the vertex $(n,z)$ of the simplex,
    $
        \gamma^W_{n,z}=\tfrac{d}{d\varepsilon}W(q_\varepsilon)\big|_{0}=\partial_{n,z}W-\mathbb E_{q_{\mathrm{train}}}[\partial W].
    $
    If $q_{\mathrm{train}}$ maximizes $W$ over the simplex of reweightings, the first-order (KKT) condition is that the partials $\partial_{n,z}W$ are equal---to a common multiplier $\lambda$---across $\mathrm{supp}(q_{\mathrm{train}})$; their $q_{\mathrm{train}}$-average is then also $\lambda$, so $\gamma^W_{n,z}=\lambda-\lambda=0$ for every $(n,z)\in\mathrm{supp}(q_{\mathrm{train}})$.
\end{proof}

The welfare consequences of RLHF distortion are fully characterized by the single vector $\Delta\psi \in \mathbb{R}^k$: the aggregate welfare loss is $\bar{\theta}_w^\top \Delta\psi$, and the per-agent redistribution is $(\theta_n^\top \Delta\psi)_{n=1}^N$.

\begin{remark}[Query-level decomposition]
    The impact gap decomposes across deployment queries as
    \begin{equation}
        \bar{\theta}_w^\top \Delta\psi = \mathbb E_{z \sim q_{\mathrm{dep}}}\left[(\bar{\theta}_w^\top \alpha_z)\Big(\tfrac{1}{2}\operatorname{sign}(\alpha_z^\top \bar{\theta}_w) - \big(\sigma(\alpha_z^\top \hat{\theta}^{\mathrm{RLHF}}) - \tfrac{1}{2}\big)\Big)\right].
    \end{equation}
    The distortion on each query $z$ is the difference between the deterministic utilitarian decision and the probabilistic RLHF decision, weighted by the welfare importance of the query.
\end{remark}
\subsection{Recoverability under Anonymous RLHF}
\label{app:relabel}
A natural question is whether reweighting the training distribution can close the gap. By the impact identity (Proposition~\ref{prop:rlhf-impact}), reweighting moves the deployed impact only through the per-query label frequencies $\{\mu_z\}$, and the answer depends on whether the reweighting conditions on individual identities.

Let $C=\operatorname{conv}(\{\theta_n\}_{n=1}^N)$ be the convex hull of the population preferences, let $S = \operatorname{span}\{\alpha_z : z \in \operatorname{supp}(q_{\mathrm{train}})\}$, and let $P_S$ denote projection onto $S$.

\begin{definition}
    \label{def:recoverable-set}
    A parameter $\theta\in\mathbb R^k$ is \textbf{individually mix-recoverable} if there exists, for each query $z\in Z$, a probability vector $\pi(\cdot\mid z)\in\Delta_N$ such that
    \begin{equation*}
        \sigma(\alpha_z^\top\theta)
        =
        \sum_{n=1}^N \pi(n\mid z)\,\sigma(\alpha_z^\top\theta_n)
        \qquad \text{for every $z\in Z$}
    \end{equation*}
    Let $\Theta^{\mathrm{indiv}}$ denote the set of all such parameters. A parameter is \textbf{anonymously mix-recoverable} if $\pi(n\mid z)=\frac{1}{N}$ for every $n$. Denote the set of such parameters $\Theta^{\mathrm{anon}}$.
\end{definition}

\begin{proposition}
    \label{prop:recoverable-set}
    \hspace{0em}
    \begin{enumerate}
        \item A parameter $\theta \in \Theta^{\mathrm{indiv}}$ if and only if $\alpha_z^\top\theta \in \big[\min_{n}\alpha_z^\top\theta_n,\; \max_{n}\alpha_z^\top\theta_n \big]$ for every $z\in Z$.
        
        \item $P_S \, C \subseteq P_S\,\Theta^{\mathrm{indiv}}$: the convex hull of individual preferences is contained in the individually recoverable set.
    \end{enumerate}
\end{proposition}

\begin{proof}[Proof of Proposition~\ref{prop:recoverable-set}]
    
    \hspace{0em}
    \begin{enumerate}
        \item[Part 1:]
        \begin{enumerate}
            \item[$\Rightarrow$] Per the definition of $\theta \in \Theta^{\mathrm{indiv}}$, there exists a $\pi$ such that, for all $z$,
            \begin{equation}
                \sigma(\alpha_z^\top\theta) = \sum_{n=1}^N\pi(n|z) \, \sigma(\alpha_z^\top \theta_n)
            \end{equation}
            As a convex combination of scalars, the right-hand side must lie between its extreme points, so that $\sigma(\alpha_z^\top\theta) \in [\min_{n}\sigma(\alpha_z^\top\theta_n),\max_{n}\sigma(\alpha_z^\top\theta_n)]$ for all $z$.
            But then, by the strict monotonicity of $\sigma$, this implies that $\alpha_z^\top\theta \in [\min_{n}\alpha_z^\top\theta_n,\max_{n}\alpha_z^\top\theta_n]$ for all $z$.
            
            \item[$\Leftarrow$] Let $z \in Z$ be arbitrary, and suppose that $\sigma(\alpha_z^\top\theta) \in [\min_{n}\sigma(\alpha_z^\top\theta_n),\max_{n}\sigma(\alpha_z^\top\theta_n)]$. Let $\underline{n} = \arg\min_{n}\sigma(\alpha_z^\top\theta_n)$ and $\overline{n} = \arg\max_{n}\sigma(\alpha_z^\top\theta_n)$. Since $\sigma(\alpha_z^\top\theta)$ lies in the closed interval with endpoints $\sigma(\alpha_z^\top\theta_{\underline{n}})$ and $\sigma(\alpha_z^\top\theta_{\overline{n}})$, there exists $t_z \in [0,1]$ such that            %
            \begin{equation}
                \sigma(\alpha_z^\top\theta) = (1-t_z)\,\sigma(\alpha_z^\top\theta_{\underline{n}}) + t_z\,\sigma(\alpha_z^\top\theta_{\overline{n}}).
            \end{equation}
            Set $\pi(\underline{n}\mid z) = 1-t_z$, $\pi(\overline{n}\mid z) = t_z$, and $\pi(n\mid z) = 0$ for every other $n$. If $\underline{n} = \overline{n}$, the interval is a single point, and we set $\pi(\underline{n}\mid z) = 1$ instead. In either case $\pi(\cdot\mid z) \in \Delta_N$, as the definition of $\Theta^{\mathrm{indiv}}$ requires. Applying this construction at each $z \in Z$ yields the result.
        \end{enumerate}
        
        \item[Part 2:] $P_S C\subseteq P_S\Theta^{\mathrm{indiv}}$.
Fix any $\theta\in C$. By definition of $C$, there exists $\lambda\in\Delta_N$ such that
$
    \theta=\sum_{n=1}^N\lambda_n\theta_n.
$
Then, for every $z\in Z$,
$
    \alpha_z^\top\theta
    =
    \sum_{n=1}^N\lambda_n\alpha_z^\top\theta_n,
$
so $\alpha_z^\top\theta$ lies between $\min_n\alpha_z^\top\theta_n$ and $\max_n\alpha_z^\top\theta_n$. By Part 1, $\theta\in\Theta^{\mathrm{indiv}}$, and therefore $P_S C\subseteq P_S\Theta^{\mathrm{indiv}}$.

    \end{enumerate}
\end{proof}

Thus, with individual-level label weights, RLHF can recover any parameter in the convex hull $C$ of the population preferences---in particular the welfare-weighted parameter $\bar\theta_w$ for every $w\in\Delta_N$: if $\theta\in\Theta^{\mathrm{indiv}}$ with mixture weights $\pi$, then setting $q_{\mathrm{train}}(n\mid z)=\pi(n\mid z)$ makes $\theta$ satisfy the stationarity condition of Proposition~\ref{prop:rlhf-impact} exactly, with zero population loss. Anonymous weighting, by contrast, is far more rigid:

\begin{proposition}\label{prop:unif-recoverable-set}
    If $S \neq \{0\}$, then $P_S \, \Theta^{\mathrm{anon}}$ is either empty or a singleton.
\end{proposition}
\begin{proof}[Proof of Proposition~\ref{prop:unif-recoverable-set}]
    When $\theta\in\Theta^{\mathrm{anon}}$, 
    \begin{equation}
    \alpha_z^\top \theta
    =
    \sigma^{-1}\left(\frac{1}{N}\sum_{n=1}^N \sigma(\alpha_z^\top \theta_n)\right)
    \qquad \forall z\in Z
    \label{eq:unif-recovery}
    \end{equation}
    where the right-hand side is determined entirely by the population $\{\theta_n\}_{n=1}^N$ for each $z$.
    If there exists no $\theta$ satisfying this system of linear constraints, then $\Theta^{\mathrm{anon}}=\varnothing$ and hence
    $P_S\,\Theta^{\mathrm{anon}}$ is empty.
    
    Otherwise, suppose $\theta,\theta'\in\Theta^{\mathrm{anon}}$.
    They both satisfy the system of linear constraints, so for every $z\in Z$, $\alpha_z^\top \theta = \alpha_z^\top \theta'$ and hence $\alpha_z^\top(\theta-\theta')=0$ for $z\in Z$. This implies $\theta-\theta'$ is orthogonal to $\mathrm{span}\{\alpha_z:z\in Z\}\supseteq S$, and hence to $S$. Thus all elements of $\Theta^{\mathrm{anon}}$ have the same projection onto $S$. Therefore $P_S\,\Theta^{\mathrm{anon}}$ is at most a singleton.
\end{proof}

Anonymous RLHF is pinned to at most one parameter in the identifiable subspace, regardless of how the query distribution is chosen.

\begin{corollary}[Impact-space consequence]
    \label{cor:recoverable-impacts}
    Let $\Psi^{\mathrm{indiv}} = \{\psi(\theta) : \theta \in \Theta^{\mathrm{indiv}}\}$ and $\Psi^{\mathrm{anon}} = \{\psi(\theta) : \theta \in \Theta^{\mathrm{anon}}\}$. Then $\psi(\bar{\theta}_w) \in \Psi^{\mathrm{indiv}}$ for every $w \in \Delta_N$, but $\Psi^{\mathrm{anon}}$ is at most a singleton. If $\Psi^{\mathrm{anon}} = \{\psi^{\mathrm{anon}}\}$, the welfare gap from the anonymity constraint is $\bar{\theta}_w^\top(\psi^\star - \psi^{\mathrm{anon}})$.
\end{corollary}

\begin{proof}[Proof of Corollary~\ref{cor:recoverable-impacts}]
    The proof of Proposition~\ref{prop:recoverable-set}(2) shows $C\subseteq\Theta^{\mathrm{indiv}}$, and $\bar\theta_w\in C$, so $\psi(\bar\theta_w)\in\Psi^{\mathrm{indiv}}$. For the second claim, any $\theta,\theta'\in\Theta^{\mathrm{anon}}$ satisfy \eqref{eq:unif-recovery}, so $\alpha_z^\top\theta=\alpha_z^\top\theta'$ for every $z\in Z$; since $\psi$ depends on its argument only through $(\alpha_z^\top\theta)_{z\in\operatorname{supp}(q_{\mathrm{dep}})}$ and $\operatorname{supp}(q_{\mathrm{dep}})\subseteq Z$, the map $\psi$ is constant on $\Theta^{\mathrm{anon}}$, and $\Psi^{\mathrm{anon}}$ is at most a singleton.
\end{proof}

Letting $\mu_z = \mathbb E_{n \sim q_{\mathrm{train}}(n\mid z)}[\sigma(\alpha_z^\top\theta_n)]$ denote the population-level choice frequency for $z$, anonymous RLHF effectively computes the $M$-projection of the choice data onto the logistic model family:
\begin{equation*}
    \hat\theta^{\mathrm{RLHF}}_{q_{\mathrm{train}}}
    \in
    \arg\min_{\theta\in\mathbb R^k}\
    \mathbb E_{(z,n)\sim q_{\mathrm{train}}}\!\left[
        \mathrm{D}_{\mathrm{KL}}\!\left(
            \mathrm{Bern}(\mu_z)
            \,\big\|\,
            \mathrm{Bern}(\sigma(\alpha_z^\top\theta))
        \right)
    \right].
\end{equation*}

\begin{proposition}\label{prop:anonymous-recoverable-set}
    Suppose every query is represented in training, $\mathrm{supp}(q_{\mathrm{train}})=Z$. If $\hat{\theta}_{q_{\mathrm{train}}}^{\mathrm{RLHF}}$ achieves zero population training loss, then
    $
        P_S\hat{\theta}_{q_{\mathrm{train}}}^{\mathrm{RLHF}}
        \in
        P_S\Theta^{\mathrm{indiv}}.
    $
\end{proposition}

\begin{proof}[Proof of Proposition~\ref{prop:anonymous-recoverable-set}]
    If $\hat{\theta}^{\mathrm{RLHF}}_{q_{\mathrm{train}}}$ achieves zero excess loss, then every KL term vanishes on the support of $q_{\mathrm{train}}$, hence
    \begin{equation}
        \sigma(\alpha_z^\top \hat{\theta}^{\mathrm{RLHF}}_{q_{\mathrm{train}}})
        =
        \sum_{n=1}^N q_{\mathrm{train}}(n\mid z)\,\sigma(\alpha_z^\top\theta_n)
        \qquad \forall z\in\mathrm{supp}(q_{\mathrm{train}})
        \label{eq:exact-fit-on-support}
    \end{equation}
    Since $\mathrm{supp}(q_{\mathrm{train}})=Z$, this is exactly the mixture-recovery condition of Proposition~\ref{prop:recoverable-set}(1) at every $z\in Z$. As $\hat{\theta}^{\mathrm{RLHF}}_{q_{\mathrm{train}}}$ is only identifiable in the subspace $S$, applying Proposition~\ref{prop:recoverable-set} gives
$
    P_S\hat{\theta}^{\mathrm{RLHF}}_{q_{\mathrm{train}}}
    \in
    P_S\Theta^{\mathrm{indiv}}.
$
\end{proof}

\section{Additional Material for Section \ref{sec:constrained-welfare}}
\label{app:constrained-welfare}

We first define alignment externalities, then give the proofs and interpretation of constrained welfare and harm bounds. The public-spirit and participation-externality examples follow, each with its proof.

\subsection{Alignment Externalities}
\label{app:fullalignment}

Define $\Delta\psi = \psi^\star - \psi^{\mathrm{RLHF}}$ as the impact gap. Writing $U^\star=h_{\bar\Psi}(\bar\theta_w)=\bar\theta_w^\top\psi^\star$ for the optimal utilitarian welfare, the welfare loss from RLHF is
$
    U^\star - \bar{\theta}_w^\top\psi^{\mathrm{RLHF}}
    = \bar{\theta}_w^\top \Delta\psi \;\geq\; 0
$,
with equality if and only if $\psi^{\mathrm{RLHF}}$ is itself utilitarian-optimal.

Agent $n$'s welfare change from deploying $\psi^{\mathrm{RLHF}}$ instead of $\psi^\star$ is
$
    \theta_n^\top\psi^{\mathrm{RLHF}} - \theta_n^\top\psi^\star
    = -\theta_n^\top \Delta\psi.
$
Agents whose preferences are aligned with the impact gap ($\theta_n^\top \Delta\psi > 0$) are harmed by RLHF's distortion; agents anti-aligned with $\Delta\psi$ benefit.

\subsubsection{Participation Alignment Externality}

The participation alignment externality measures the welfare impact of including an agent's data in the training process relative to the counterfactual in which they are excluded. This is the alignment analogue of the VCG pivot from mechanism design: it quantifies how much an agent's participation affects others.

To see this formally, let $\hat{\theta}$ denote the parameter obtained from training on the full population and $\hat{\theta}_{-n}$ the parameter obtained when agent $n$ is excluded (and remaining agents' weights are renormalized). The corresponding deployed impacts are $\psi = \psi(\hat{\theta})$ and $\psi_{-n} = \psi(\hat{\theta}_{-n})$.

\begin{definition}[Participation Alignment Externality]
    \label{def:alignment-externality}
    The participation alignment externality of agent $n$ on agent $m$ is
    \begin{equation}
        \Gamma_{n,m}
        = U_m(\hat{\theta}) - U_m(\hat{\theta}_{-n})
        = \theta_m^\top\!\left(\psi - \psi_{-n}\right).
    \end{equation}
    The second equality follows from $U_m(\theta) = \theta_m^\top \psi(\theta)$. In impact space, the alignment externality is the inner product of $m$'s preference with the impact shift $\Delta_n = \psi - \psi_{-n}$ caused by $n$'s participation.
\end{definition}

The impact shift $\Delta_n \in \mathbb{R}^k$ is a vector that represents the welfare consequences of $n$'s participation: agent $m$ benefits from $n$'s inclusion whenever $\theta_m^\top \Delta_n > 0$ (their preferences are aligned with the impact shift) and is harmed whenever $\theta_m^\top \Delta_n < 0$. Thus, the aggregate externality of $n$ on all other agents is
\begin{equation}
    \Gamma_{n,-n}
    = \sum_{m \neq n} w_m \, \Gamma_{n,m}
    = (1-w_n)\,\bar{\theta}_{-n}^\top \Delta_n,
\end{equation}
where $\bar{\theta}_{-n} = \frac{1}{1-w_n}\sum_{m \neq n} w_m \theta_m$ is the leave-one-out utilitarian preference. This is positive if the addition of $n$'s data moves the estimated preference parameter in a direction the rest of the population likes and negative when $n$'s participation moves the parameter in a direction the rest of the population dislikes.

\subsection{Constrained Welfare and Harm Bounds}

\begin{proof}[Proof of Theorem~\ref{thm:constrained-linear-social-choice}]
    Define the Lagrangian
    \begin{equation}
        \mathcal L(\psi,\eta)
        =
        \bar{\theta}_w^\top \psi
        +
        \sum_{n=1}^N
        \eta_n(\gamma_n^\top\psi - c_n),
        \qquad
        \eta_n \geq 0.
    \end{equation}
    Equivalently,
    \begin{equation}
        \mathcal L(\psi,\eta)
        =
        \left(
            \bar{\theta}_w
            +
            \sum_{n=1}^N \eta_n\gamma_n
        \right)^\top \psi
        -
        \sum_{n=1}^N \eta_n c_n.
    \end{equation}
    Since the feasible set is non-empty and compact and the objective is continuous,
    an optimum exists.

    By Theorem~\ref{thm:welfare-geometry}, $\bar{\Psi}$ is a zonotope generated by finitely many segments, hence a polytope, so the feasible set $\bar\Psi\cap\bigcap_n\{\gamma_n^\top\psi\geq c_n\}$ is a polyhedron and the problem is a finite linear program. Linear-programming duality therefore yields---with no constraint qualification required---multipliers $\eta_n \geq 0$ such that the constrained optimum also solves
    \begin{equation}
        \psi^{\mathrm{wel}}
        \in
        \argmax_{\psi \in \bar{\Psi}}
        \left(
            \bar{\theta}_w
            +
            \sum_{n=1}^N \eta_n\gamma_n
        \right)^\top \psi.
    \end{equation}
    Complementary slackness gives
    \begin{equation}
        \eta_n(\gamma_n^\top\psi^{\mathrm{wel}} - c_n)=0
        \qquad \forall n.
    \end{equation}
    Hence only binding constraints enter the effective objective direction, so if
    $\mathcal B=\{n:\eta_n>0\}$, then
    \begin{equation}
        \bar{\theta}_w
        +
        \sum_{n\in\mathcal B}\eta_n\gamma_n
    \end{equation}
    is the effective welfare direction. Finally, let $V(c)$ denote the optimal value as a
    function of the threshold vector $c$. Raising any $c_n$ shrinks the feasible set, so $V$
    is weakly decreasing in each $c_n$. For concavity, let $\psi_1,\psi_2$ be optimal at
    thresholds $c^{(1)},c^{(2)}$ and $t\in[0,1]$; then $t\psi_1+(1-t)\psi_2\in\bar\Psi$ by
    convexity and satisfies $\gamma_n^\top(t\psi_1+(1-t)\psi_2)\geq tc^{(1)}_n+(1-t)c^{(2)}_n$
    for all $n$, so it is feasible at $tc^{(1)}+(1-t)c^{(2)}$ and
    $V(tc^{(1)}+(1-t)c^{(2)})\geq tV(c^{(1)})+(1-t)V(c^{(2)})$. Whenever $V$ is
    differentiable, its derivative satisfies
    \begin{equation}
        \frac{\partial V}{\partial c_n}=-\eta_n.
    \end{equation}
\end{proof}

\begin{proof}[Proof of Corollary~\ref{cor:bounded-harm}]
    Apply Theorem~\ref{thm:constrained-linear-social-choice} with
    $\gamma_n=\theta_n$ and $c_n=b_n$. The effective welfare direction is
    \begin{equation*}
        \bar{\theta}_w+\sum_{n=1}^N\eta_n\theta_n
        =
        \sum_{n=1}^N(w_n+\eta_n)\theta_n.
    \end{equation*}
    Complementary slackness gives
    \begin{equation*}
        \eta_n(\theta_n^\top\psi_b^{\mathrm h}-b_n)=0
        \qquad \forall n.
    \end{equation*}
\end{proof}

The constrained welfare mechanisms introduced in Section \ref{sec:constrained-welfare} can be understood as directly regulating the alignment externality system defined in Appendix~\ref{app:fullalignment}.

\begin{proposition}[Externality Decomposition of Welfare Floors]
    \label{prop:externality-decomposition}
    Let $\psi=\psi(\hat\theta)$ be the impact deployed by the full-population training procedure and
    $\psi_{-n} = \psi(\hat{\theta}_{-n})$ the leave-one-out impact, as in Definition~\ref{def:alignment-externality}.
    Then the welfare floor constraint
    $\theta_m^\top \psi \geq b_m$ is equivalent to
    \begin{equation}
        \Gamma_{n,m} \;\geq\; b_m - \theta_m^\top \psi_{-n}
        \qquad \forall\, n.
    \end{equation}
    That is, the welfare floor constrains each agent's participation
    externality on $m$ to be large enough to lift $m$'s welfare above
    the floor, given $m$'s baseline welfare without $n$.
\end{proposition}

\begin{proof}
    From $\theta_m^\top \psi = \theta_m^\top \psi_{-n} + \Gamma_{n,m}$,
    the floor $\theta_m^\top \psi \geq b_m$ is equivalent to
    $\Gamma_{n,m} \geq b_m - \theta_m^\top \psi_{-n}$.
\end{proof}

Figure~\ref{fig:welfaredist} shows how each mechanism distributes welfare across individuals at high choice precision. Under the utilitarian, RLHF, and strategyproof mechanisms the distribution is wide and carries mass below zero---agents who are actively harmed. The harm-bounded mechanism removes this left tail: as the welfare floor $b$ tightens toward $0$, more agents' constraints bind and the distribution compresses upward against the floor, at the cost of some reduction in mean welfare. Appendix~\ref{app:welfaresweep} traces the same effect as a function of the choice-precision parameter $\beta$.

A harm-bounded mechanism allows us to regulate individuals' participation externalities (Figure~\ref{fig:rlhfvsbh_peragent} in Appendix~\ref{app:peragentharm}). The bright rows in the harm-bounded panel are the agents for whom the constraint was binding---agents whose welfare would otherwise fall below the floor. When such an agent is removed from the dataset, their binding constraint disappears and everyone else benefits from returning closer to the utilitarian regime; hence, their inclusion imposes a drastic negative externality on every other member of the population.

\subsection{Additional Example: Public Spirit}
\label{app:publicspirit}

Another example of a welfare constraint is \textit{public spirit}~\citep{flanigan2023distortion}. This constraint allows people to tolerate personal harm in proportion to the social benefit created, where $\gamma \in [0,1]$ is an individual's degree of public spirit.\footnote{This is a special case of Theorem~\ref{thm:constrained-linear-social-choice} with $\gamma_n = \gamma\bar{\theta}_w + (1-\gamma)\theta_n$ and $c_n = (\gamma\bar{\theta}_w + (1-\gamma)\theta_n)^\top\psi_0$, where $\psi_0$ is a reference impact (e.g., the status quo or the random model).}

\begin{definition}[Public-Spirit Alignment]
    For a public-spirit parameter $\gamma \in [0,1]$ and reference impact $\psi_0$:
    \begin{equation*}
        \psi_\gamma^{\mathrm{ps}} \in \argmax_{\psi \in \bar{\Psi}}
        \quad \bar{\theta}_w^\top \psi
        \qquad \textrm{s.t.} \quad
        (\gamma\,\bar{\theta}_w + (1-\gamma)\theta_n)^\top \psi
        \geq
        (\gamma\,\bar{\theta}_w + (1-\gamma)\theta_n)^\top \psi_0
        \quad \forall\, n.
        \label{eq:public-spirit}
    \end{equation*}
\end{definition}

The constraint has a direct interpretation, saying that agent $n$ will accept some personal loss only when it is justified by enough social gain:
    $(1-\gamma)\,\theta_n^\top(\psi_0-\psi)
    \leq
    \gamma\,\bar{\theta}_w^\top(\psi-\psi_0).$
When $\gamma=0$, this reduces to an individual rationality constraint: no agent
can be made worse off relative to $\psi_0$.

\begin{corollary}[of Theorem~\ref{thm:constrained-linear-social-choice}]
    \label{cor:public-spirit}
    The public-spirit optimum satisfies
    \begin{equation}
        \psi_\gamma^{\mathrm{ps}}
        \in
        \argmax_{\psi \in \bar{\Psi}}
        \quad
        \left(
            \left(1 + \gamma\sum_{n \in \mathcal{B}} \eta_n\right)\bar{\theta}_w
            +
            (1-\gamma)\sum_{n \in \mathcal{B}}\eta_n\,\theta_n
        \right)^\top \psi,
    \end{equation}
    where $\eta_n \geq 0$ and
    $
        \eta_n
        \left[
            \big(\gamma\bar{\theta}_w + (1-\gamma)\theta_n\big)^\top
            (\psi_\gamma^{\mathrm{ps}}-\psi_0)
        \right]
        =
        0
        \qquad \forall n.
    $
\end{corollary}

\begin{proof}[Proof of Corollary~\ref{cor:public-spirit}]
    Apply Theorem~\ref{thm:constrained-linear-social-choice} with
    \begin{equation*}
        \gamma_n^{\mathrm{ps}}
        =
        \gamma\bar{\theta}_w+(1-\gamma)\theta_n
        \qquad\text{and}\qquad
        c_n
        =
        \big(\gamma\bar{\theta}_w+(1-\gamma)\theta_n\big)^\top\psi_0.
    \end{equation*}
    The effective welfare direction is
    \begin{align*}
        \bar{\theta}_w
        +
        \sum_{n=1}^N\eta_n
        \big(\gamma\bar{\theta}_w+(1-\gamma)\theta_n\big)
        &=
        \left(1+\gamma\sum_{n=1}^N\eta_n\right)\bar{\theta}_w
        +
        (1-\gamma)\sum_{n=1}^N\eta_n\theta_n.
    \end{align*}
    By complementary slackness, only binding constraints enter this expression, giving
    \begin{equation*}
        \left(1+\gamma\sum_{n\in\mathcal B}\eta_n\right)\bar{\theta}_w
        +
        (1-\gamma)\sum_{n\in\mathcal B}\eta_n\theta_n.
    \end{equation*}
    The complementary slackness condition is
    \begin{equation*}
        \eta_n
        \left[
            \big(\gamma\bar{\theta}_w+(1-\gamma)\theta_n\big)^\top
            (\psi_\gamma^{\mathrm{ps}}-\psi_0)
        \right]
        =
        0
        \qquad \forall n.
    \end{equation*}
\end{proof}

Public spirit changes how binding agents pull on the solution. When $\gamma$ is small,
their constraints act mostly like personal welfare protections, tilting the objective
toward their own preferences. When $\gamma$ is large, the same constraints put more
weight on the social objective itself. Thus public spirit makes the constrained solution
look less like individualized compensation and more like utilitarian alignment.

\subsection{Additional Example: Bounding Participation Externalities}

The welfare floor (Corollary~\ref{cor:bounded-harm}) protects agents who are \emph{harmed}---it increases their influence over the deployed model. Another potential concern is limiting the harm that any single agent's participation \emph{imposes on everyone else}. We formally define the participation externality in Appendix \ref{app:fullalignment}. Here we need only the \emph{leave-one-out utilitarian preference}, the welfare-weighted average preference of everyone other than $n$,
\begin{equation}
    \bar{\theta}_{-n} = \frac{1}{1-w_n}\sum_{m \neq n} w_m \theta_m,
\end{equation}
which is well defined whenever $w_n<1$.

\begin{definition}[Externality-bounded alignment]
    For a reference impact $\psi_0$ and tolerances $b \in \mathbb{R}_{\geq 0}^N$:
    $
        \psi_b^{\mathrm{ext}}
        \in
        \argmax_{\psi \in \bar{\Psi}}
        \;
        \bar{\theta}_w^\top\psi
        \quad
        \textrm{s.t.}
        \quad
        \bar{\theta}_{-n}^\top\psi
        \geq
        \bar{\theta}_{-n}^\top\psi_0 - b_n
        \quad \forall\, n.
    $
\end{definition}

The constraint says: the aggregate welfare of everyone except $n$, evaluated at the deployed impact, must not fall more than $(1-w_n)\,b_n$ below its value at the reference (recall that $\bar{\theta}_{-n}$ normalizes the leave-one-out weights by $1/(1-w_n)$, which requires $w_n<1$ for all $n$). This bounds the cost that accommodating $n$'s preferences imposes on the rest of the population. When $b_n = 0$, the mechanism cannot reduce the rest of the population's aggregate welfare by including $n$; as $b_n \to \infty$, the constraint becomes vacuous.

\begin{corollary}[of Theorem~\ref{thm:constrained-linear-social-choice}]
    \label{cor:externality-bound}
    The externality-bounded optimum satisfies
    \begin{equation}
        \psi_b^{\mathrm{ext}}
        \in
        \argmax_{\psi \in \bar{\Psi}}
        \;
        \left(
            \left(1 + \sum_{n \in \mathcal{B}}
            \frac{\eta_n}{1 - w_n}\right)\bar{\theta}_w
            \;-\;
            \sum_{n \in \mathcal{B}}
            \frac{w_n\,\eta_n}{1 - w_n}\,\theta_n
        \right)^\top\psi,
    \end{equation}
    where $\eta_n \geq 0$ and
    $\eta_n(\bar{\theta}_{-n}^\top\psi_b^{\mathrm{ext}}
    - \bar{\theta}_{-n}^\top\psi_0 + b_n) = 0$ for
    all $n$.
\end{corollary}

\begin{proof}[Proof of Corollary~\ref{cor:externality-bound}]
    Apply Theorem~\ref{thm:constrained-linear-social-choice}
    with $\gamma_n = \bar{\theta}_{-n}$. Substituting
    $\bar{\theta}_{-n} = \frac{1}{1 - w_n}
    (\bar{\theta}_w - w_n\theta_n)$ into the effective
    direction:
    \begin{align*}
        \bar{\theta}_w
        + \sum_{n \in \mathcal{B}} \eta_n\,\bar{\theta}_{-n}
        &=
        \bar{\theta}_w
        + \sum_{n \in \mathcal{B}}
        \frac{\eta_n}{1 - w_n}
        (\bar{\theta}_w - w_n\,\theta_n)
        \\
        &=
        \left(1 + \sum_{n \in \mathcal{B}}
        \frac{\eta_n}{1 - w_n}\right)\bar{\theta}_w
        \;-\;
        \sum_{n \in \mathcal{B}}
        \frac{w_n\,\eta_n}{1 - w_n}\,\theta_n.
    \end{align*}
    Complementary slackness follows from the general theorem.
\end{proof}

The effective direction shifts weight in the opposite direction from the welfare floor. In harm-bounded alignment (Corollary~\ref{cor:bounded-harm}), binding agents receive \emph{additional} weight: $w_n \to w_n + \eta_n$. Here, binding agents have their weight \emph{relatively} reduced. Collecting the coefficient on $\theta_n$ for a binding agent $n$, the amplification of $\bar{\theta}_w$ contributes $w_n\eta_n/(1-w_n)$ and the subtraction removes exactly the same amount, so agent $n$'s effective weight is $w_n\big(1+\sum_{m \in \mathcal{B},\, m \neq n}\eta_m/(1-w_m)\big)$, while a non-binding agent's is $w_n\big(1+\sum_{m \in \mathcal{B}}\eta_m/(1-w_m)\big)$. No agent's weight falls below $w_n$ (the effective direction is $\bar{\theta}_w+\sum_{n\in\mathcal B}\eta_n\bar{\theta}_{-n}$, a positive combination), but binding agents are weighted less than everyone else. Intuitively, the welfare floor asks ``who is being harmed?'' and gives them more voice; the externality bound asks ``who is causing harm?'' and attenuates their influence relative to the rest of the population.

With equal weights $w_n = 1/N$ and large $N$, the effective direction simplifies to approximately $(1 + \sum_{n \in \mathcal{B}}\eta_n)\,\bar{\theta}_w - \frac{1}{N}\sum_{n \in \mathcal{B}}\eta_n\,\theta_n$: the dominant effect is an amplification of the utilitarian direction, with a small per-agent correction of order $1/N$.


\section{Additional Experiments}
\label{app:empirics}

\subsection{Dataset Information and Experimental Setup}
All experiments were run locally on commodity hardware. An A100 was rented to compute the feature embeddings for the Community Alignment dataset, which took approximately 5 minutes to run. All code can be found at \url{https://github.com/michelleeesi/hiddenstructure}.

\subsubsection{412 Food Rescue}
We use the data from \cite{lee2019webuildai}, which records pairwise choices over food-rescue recipients along $k=7$ features (\textit{size, access, income, poverty, last\_donation, total\_donation, dist}). Each row of the CSV is one binary comparison by one participant: we set $\alpha_z = \phi_A - \phi_B$ and $y_z = \mathbf 1\{\text{chose }A\}$, with 19 participants (different stakeholders for the 412 Food Rescue program) and 45 pairwise responses each. This data is not public and was accessed with permission from the original authors of the WeBuildAI paper.

\subsubsection{Kidney Exchange}
We use the kidney-allocation data from \cite{keswani2024pros, boerstler2024stability}, in which respondents choose between two patients described by $k=5$ features (\textit{elderlyDep, lifeYearsGained, obesity, weeklyWorkhours, yearsWaiting}). We construct $\alpha_z$ from the released columns as the left-option minus right-option feature difference and $y_z$ from the \texttt{chosen} indicator. Each respondent answered $\sim 400$ pairwise comparisons. The data is publicly available online at \url{https://github.com/vijaykeswani/Preference-Instability/tree/main/Study%201%262%20-%20AIES%202024}.

\subsubsection{Moral Machine}
We use the trolley-problem dataset from~\cite{awad2018moral}, which records, for each of millions of participants, a choice between two groups of characters drawn from 20 character types (\textit{Man, Woman, Pregnant, OldMan, Boy, Girl, Homeless, Criminal, MaleExecutive, FemaleExecutive, MaleAthlete, FemaleAthlete, MaleDoctor, FemaleDoctor, Dog, Cat,} etc.; $k=20$). We process the public data by restricting to annotators with between 100 and 500 responses, and by screening out bots. The Moral Machine dataset is publicly available online at \url{https://osf.io/mxa6z/overview}.

\subsubsection{Community Alignment}
We use the Community Alignment Dataset of~\cite{zhang2025cultivating}, which contains over 200,000 pairwise preference judgments collected from over 3000 annotators across five countries. We apply the natural language preference learning method of \citet{wojtowicz2026weights}. Options are encoded by a frozen sentence encoder (Qwen3-Embedding-8B) and pairwise differences are projected onto a $K=15$-dimensional human-interpretable basis of preference-relevant axes of variation in the choice domain. Per-participant preferences are then estimated under a Bradley--Terry--Luce model on this subspace. We filtered for annotators with $\geq 20$ responses, leaving $N = 2387$ annotators as agents.

\subsection{Robustness to Nonlinear Utilities}
\label{app:linearity}

Our framework assumes that each individual's utility is linear in a fixed feature map, $u_n(x) = \varphi(x)^\top\theta_n$. It does not assume that utility is linear in the raw attributes of an outcome. If $\varphi$ is replaced with a richer map $\tilde\varphi$ that includes nonlinear combinations of the attributes, the result is a more expressive utility function that the framework handles without change. We redefine the impact vector, the impact polytope, and the participation externality construction using $\tilde\varphi$, and every result carries over. Our baseline analyses use the raw attributes as features, which amounts to assuming that preferences are linear in those attributes. In this section we measure how sensitive those analyses are to that choice by augmenting the features with second-order terms. We ask two questions. First, does a linear or a nonlinear utility fit human choices better on our datasets? Second, where a nonlinear utility fits better, do our welfare and externality conclusions change?

\textbf{The nonlinear models.} For a comparison between options $A$ and $B$ with feature difference $\alpha = \varphi_A - \varphi_B$, the linear model is the Bradley--Terry choice probability $\sigma(\theta^\top\alpha)$. The quadratic model adds a squared term for each feature, $\sigma\big(\theta^\top\alpha + \gamma^\top(\varphi_A^{2} - \varphi_B^{2})\big)$. This lets the return to each feature increase or diminish while utility stays additive across features. For Moral Machine we also fit an interaction model, which adds all $\binom{k}{2}$ products $\varphi_{A,i}\varphi_{A,j} - \varphi_{B,i}\varphi_{B,j}$. It captures value that depends on which features occur together. This is the harder test, because the quadratic model cannot represent it. Every model is a logistic regression with no intercept and $\ell_2$ penalty $C=1$, matching our main analyses, on features standardized across all options in the dataset. Moral Machine and Community Alignment record only the preferred option, so for those datasets we add each training comparison in both orientations.

\textbf{Experiment 1: Does a nonlinear utility fit better?} We fit each model separately for every agent and compare the models by 5-fold cross-validation on that agent's own comparisons. Agents are those with at least 20 comparisons, except that Kidney Allocation uses the 77 agents with at least 100 comparisons and Moral Machine uses a random sample of 200 annotators. Table~\ref{tab:linearity-fit} reports held-out log loss and accuracy averaged across agents.

\begin{table}[h]
  \centering
  \caption{Held-out fit of linear and quadratic utilities, from per-agent 5-fold cross-validation averaged across agents. Lower log loss is better.}
  \label{tab:linearity-fit}
  \begin{tabular}{@{}lccccc@{}}
    \toprule
    Dataset & Comparisons & Log loss, & Log loss, & Accuracy, & Accuracy, \\
            & per agent   & linear    & quadratic & linear    & quadratic \\
    \midrule
    Kidney Allocation   & 480 & 0.285 & 0.263 & 87.8\% & 89.1\% \\
    412 Food Rescue     & 45  & 0.447 & 0.542 & 79.5\% & 77.6\% \\
    Moral Machine       & 69  & 1.005 & 2.024 & 60.1\% & 58.6\% \\
    Community Alignment & 57  & 0.512 & 0.544 & 74.5\% & 77.3\% \\
    \bottomrule
  \end{tabular}
\end{table}

Kidney Allocation is the only dataset where the quadratic utility fits better on both log loss and accuracy, and it helps 56 of 77 agents. In this dataset, 23\% of comparisons repeat a question the agent has already seen. Agents choose differently on 6.9\% of these identical comparisons, so the best any model of an individual can do on them is 93.1\% accuracy. The linear and quadratic models reach 89.9\% and 92.5\%. The quadratic utility therefore predicts human behavior better here, and comes close to the limit set by the agents' own consistency.

Elsewhere the nonlinear models overfit at the amount of data we have per agent. On 412 Food Rescue the quadratic utility helps only 3 of 19 agents. On Moral Machine it helps 1 of 200. The interaction model is no better, with log loss 1.976 and accuracy 58.1\%, so even structure that a quadratic model cannot represent does not improve prediction. On Community Alignment the quadratic utility improves accuracy but worsens log loss, and it helps about half the agents (1{,}079 of 2{,}387). This is what we expect, because those features are already a learned nonlinear projection of text.

\textbf{Experiment 2: Where utilities are nonlinear, do our results change?} A quadratic utility is a linear utility over the augmented basis $\tilde\varphi = [\varphi, \varphi^2]$. We therefore re-ran the whole pipeline in that space, including the impact vectors, the impact polytope and the participation externality matrix, and compared the results with the linear ones. We also report the welfare regret of utilitarian deployment. This is the share of welfare given up by deploying as if utilities were linear when they are really quadratic, computed for deterministic deployments ($\beta \to \infty$). Table~\ref{tab:linearity-estimands} reports the results.

\begin{table}[h]
  \centering
  \caption{Agreement between results computed under linear and quadratic utilities. Correlations are Pearson; the externality correlation uses all off-diagonal entries of $\Gamma$. Split-half reliability is the correlation between $\Gamma$ estimated from two random halves of each agent's comparisons under linear utilities. Regret is the share of quadratic-utility welfare lost by deploying the linear utilitarian decisions.}
  \label{tab:linearity-estimands}
  \begin{tabular}{@{}lcccc@{}}
    \toprule
    Dataset & Welfare     & Externality & Split-half reliability & Welfare \\
            & correlation & correlation & of externalities       & regret  \\
    \midrule
    Kidney Allocation   & 0.94 & 0.97 & 0.90 & 1.1\% \\
    412 Food Rescue     & 0.95 & 0.92 & 0.46 & 0.9\% \\
    Moral Machine       & 0.91 & 0.58 & 0.08 & 1.7\% \\
    Community Alignment & 0.79 & 0.47 & 0.39 & 4.0\% \\
    \bottomrule
  \end{tabular}
\end{table}

Kidney Allocation is the most informative test, because it is the only dataset where the quadratic utility fits better. There, per-agent welfare correlates at 0.94 and participation externalities at 0.97, the best-off and worst-off agents are identical, and the welfare regret is about 1\%. The two utilities make different choices only on queries where their utilitarian means disagree about which option is better, and they produce the same welfare everywhere else. A 1\% regret therefore means they rarely disagree, and little welfare is lost when they do.

On the other datasets, welfare regret is at most 4.0\% (Community Alignment), and per-agent welfare correlates between 0.79 and 0.95. The externality correlations are lower on Moral Machine (0.58) and Community Alignment (0.47), but this reflects how precisely externalities can be measured on those datasets, not a difference between the utility models. When we estimate the externality matrix from two random halves of each agent's data, the two estimates correlate at only 0.39 on Community Alignment and 0.08 on Moral Machine, even when both use linear utilities. Switching to quadratic utilities therefore perturbs the externalities no more than redrawing the data does. On Kidney Allocation, where externalities are precisely estimated (split-half reliability 0.90), the two utility models agree closely.

On three of the four datasets the linear model is the better estimator. Either the nonlinear models overfit at realistic per-agent sample sizes, or the features are already nonlinear. On the fourth, where a nonlinear utility does fit better, our framework absorbs it as a linear utility over a larger basis, and the welfare and externality conclusions barely move. The framework is therefore robust to the nonlinear utilities we test, and using them does not change what we find. The script \texttt{scripts/linearity\_robustness.py} in our code repository reproduces every number in this appendix.

\subsection{Welfare Dispersion vs.\ Choice Precision}
\label{app:welfaresweep}

\begin{figure}[H]
    \centering
    \includegraphics[width=\linewidth]{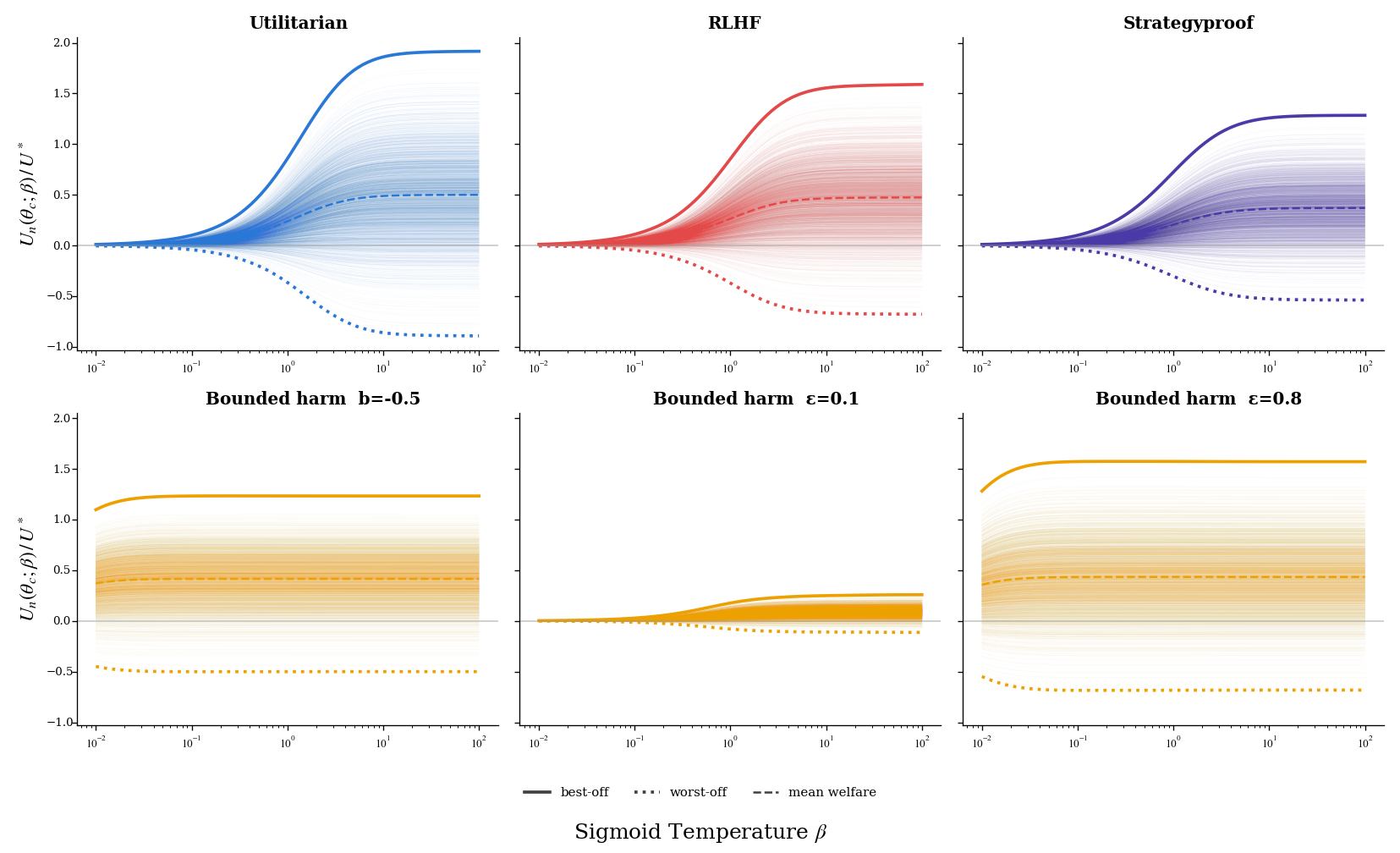}
    \caption{Community Alignment dataset: per-agent welfare $U_n/U^\star$ as a function of the choice-precision parameter $\beta$, for each mechanism. The shaded band is the density of all 2387 agents; solid, dashed, and dotted lines mark the best-off agent, mean welfare, and worst-off agent. As $\beta$ grows, choices become more deterministic, individual preferences become more powerful, and the welfare distribution spreads---widening the gap between best- and worst-off agents. Strategyproof and harm-bounded mechanisms produce lower welfare dispersion, and the harm bounds are empirically satisfied.}
    \label{fig:constrainedwelfare}
\end{figure}

\subsection{Per-Agent Harm Bound}
\label{app:peragentharm}

\begin{figure}[H]
    \centering
    \includegraphics[width=\linewidth]{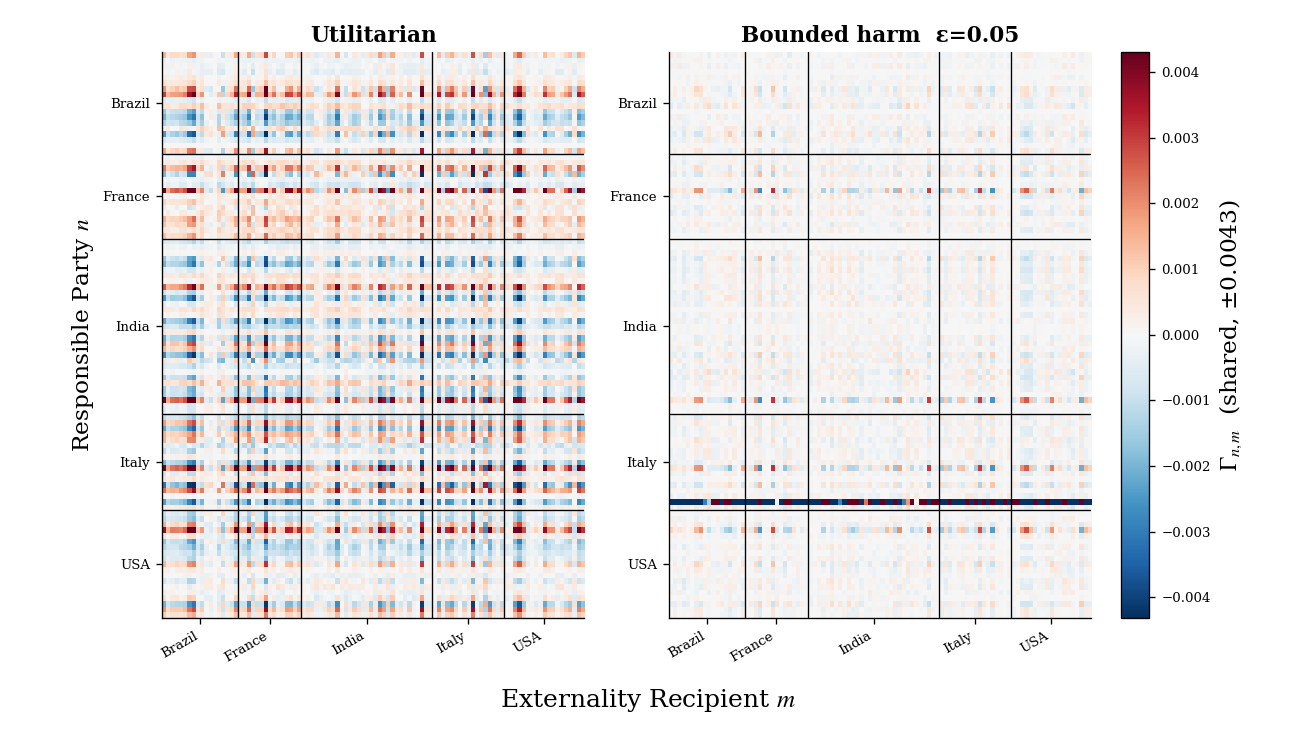}
    \caption{Participation externalities under the \emph{per-agent} harm bound. Here each agent's welfare floor is a fixed fraction of their own maximum achievable welfare, $b_n=-\epsilon\,h_{\bar\Psi}(\theta_n)$ with $\epsilon=0.05$ (lose at most 5\% of one's own best case), rather than a single absolute floor. The left panel is the utilitarian baseline; the right is the harm-bounded model. The qualitative story matches that of the single absolute floor discussed in Section~\ref{sec:constrained-welfare}: a small set of binding agents accounts for almost all of the participation externalities.}
    \label{fig:rlhfvsbh_peragent}
\end{figure}

\subsection{Person-on-Person Participation Externalities}
\label{app:fullalignmentgraphic}

\begin{figure}[H]
    \centering
    \includegraphics[width=\linewidth]{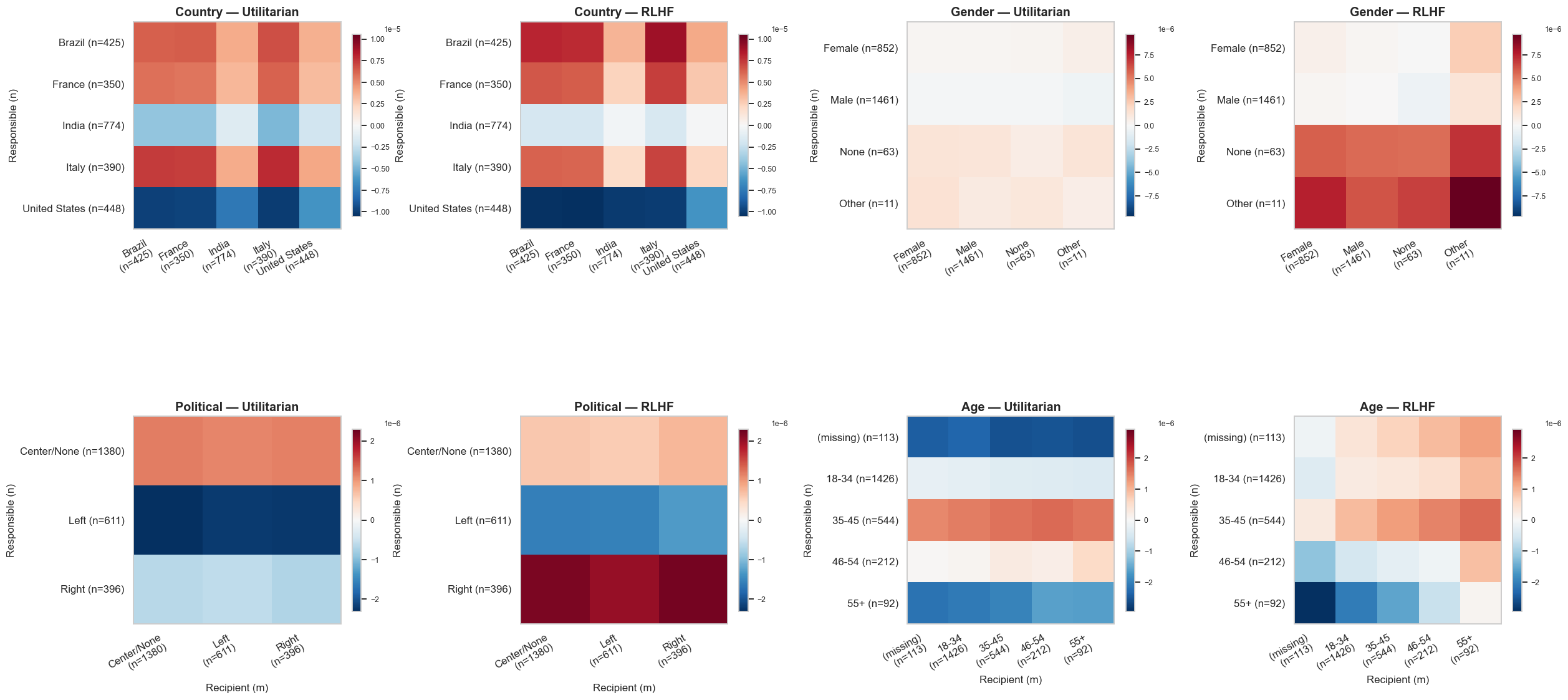}
    \caption{Community Alignment Dataset: Utilitarian vs.\ RLHF participation alignment externalities by demographic. The figure includes 2387 annotators, with participation externalities averaged within each group.}
    \label{fig:avgextdemo}
\end{figure}

\begin{figure}[H]
    \centering
    \includegraphics[width=0.7\linewidth,height=0.8\textheight,keepaspectratio]{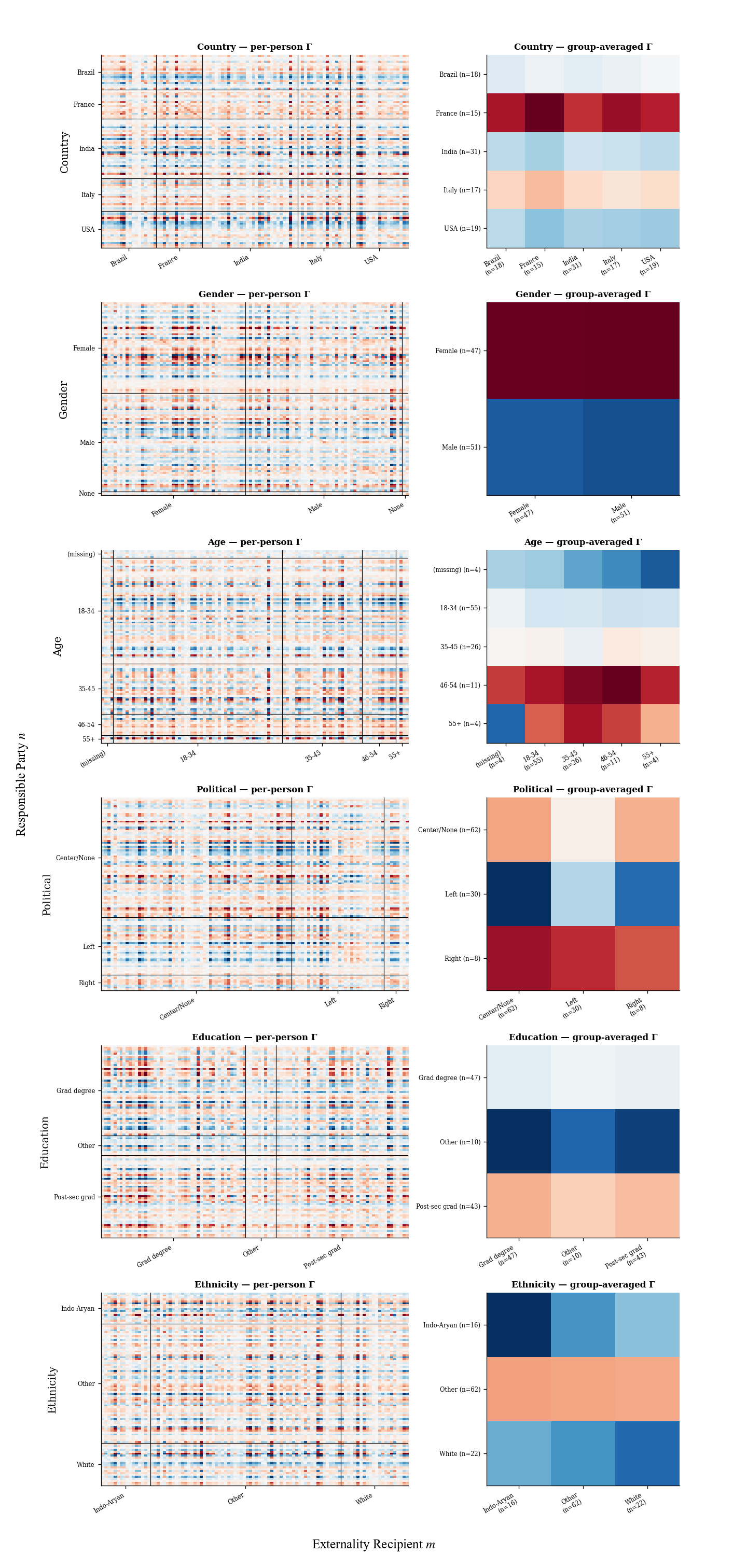}
    \caption{Community Alignment Dataset: 100-person sample with granular person-on-person participation externalities compared to the group-averaged participation externalities. Our participation externality formalization gives us a microscope to analyze where the per-individual welfare effects are coming from.}
    \label{fig:perperson}
\end{figure}

\subsection{Four Dataset Comparisons}

\begin{figure}[H]
    \centering
    \includegraphics[scale=0.35]{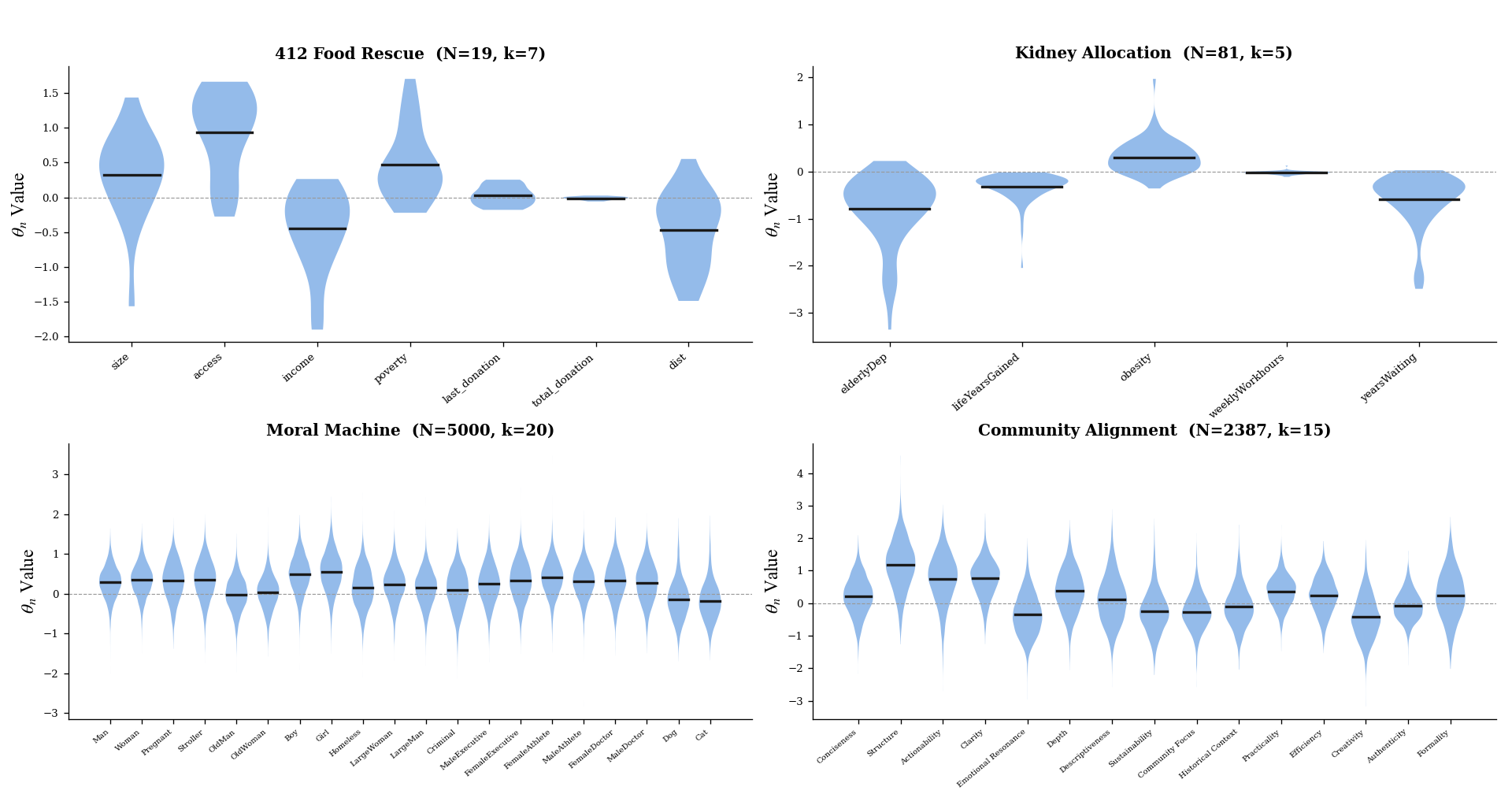}
    \caption{Preference distribution per feature dimension for four datasets. Moral Machine and Community Alignment have high disagreement, with a high number of preferences on either side of 0, while Kidney Allocation and 412 Food Rescue have less disagreement: people disagree on the magnitude of preference, but the direction is largely shared.}
    \label{fig:preferencedist}
\end{figure}

The empirical welfare losses of the uniform random dictatorship, together with the average majority margin of Proposition~\ref{prop:sp-cost}, are reported in Table~\ref{tab:sp-cost}.

\begin{table}[H]
\centering
\caption{Relative welfare loss of the uniform random dictatorship, $1-\mathbb E[U(\mu^{\mathrm{sp}})]/U^\star$, against the average majority margin $M$ (Proposition~\ref{prop:sp-cost}).}
\label{tab:sp-cost}
\begin{tabular}{lcc}
\toprule
Dataset & Relative welfare loss & Average majority margin \\
\midrule
Kidney Allocation   & 15.7\% & 0.66 \\
Community Alignment & 26.0\% & 0.60 \\
412 Food Rescue     & 31.0\% & 0.52 \\
Moral Machine       & 41.8\% & 0.41 \\
\bottomrule
\end{tabular}
\end{table}

\begin{proposition}[Population welfare guarantee]\label{prop:population}
    Suppose $\theta_1,\dots,\theta_N$ are drawn i.i.d.\ from a distribution $P$ on $\mathbb R^k$ with mean $\mu_P$ and finite second moments, and let $w_n=1/N$. Population welfare at impact $\psi$ is $\mathbb E_{\theta\sim P}[\theta^\top\psi]=\mu_P^\top\psi$, so the population problem is \eqref{eq:social-choice-problem} with $\bar\theta_w$ replaced by $\mu_P$. If $\hat\psi\in\argmax_{\psi\in\bar\Psi}\bar\theta_w^\top\psi$ is the sample optimum, then
    \begin{equation}
        \max_{\psi\in\bar\Psi}\mu_P^\top\psi-\mu_P^\top\hat\psi
        \le 2\,\|\bar\theta_w-\mu_P\|\sup_{\psi\in\bar\Psi}\|\psi\|
        \le \|\bar\theta_w-\mu_P\|\;\mathbb E_{z\sim q_{\mathrm{dep}}}\big[\|\alpha_z\|\big],
    \end{equation}
    and $\mathbb E\big[\|\bar\theta_w-\mu_P\|\big]\le\sqrt{\operatorname{tr}\operatorname{Cov}_P(\theta)/N}$, so the expected regret is $O(N^{-1/2})$.
\end{proposition}

\begin{proof}
    Let $\psi_P\in\argmax_{\psi\in\bar\Psi}\mu_P^\top\psi$ be a population optimum. Since $\hat\psi$ maximizes $\bar\theta_w^\top\psi$ over $\bar\Psi$, we have $\bar\theta_w^\top(\psi_P-\hat\psi)\le0$, and therefore
    \begin{equation}
        \mu_P^\top(\psi_P-\hat\psi)
        =(\mu_P-\bar\theta_w)^\top(\psi_P-\hat\psi)+\bar\theta_w^\top(\psi_P-\hat\psi)
        \le\|\mu_P-\bar\theta_w\|\,\|\psi_P-\hat\psi\|
    \end{equation}
    by the Cauchy--Schwarz inequality. Since $\psi_P,\hat\psi\in\bar\Psi$, the triangle inequality and Theorem~\ref{thm:welfare-geometry}(3) give
    \begin{equation}
        \|\psi_P-\hat\psi\|
        \le2\sup_{\psi\in\bar\Psi}\|\psi\|
        \le\mathbb E_{z\sim q_{\mathrm{dep}}}\big[\|\alpha_z\|\big],
    \end{equation}
    which establishes the first display. For the moment bound, $\bar\theta_w=\tfrac1N\sum_{n=1}^N\theta_n$ is the mean of $N$ i.i.d.\ draws from $P$, so $\mathbb E[\bar\theta_w]=\mu_P$ and $\operatorname{Cov}(\bar\theta_w)=\operatorname{Cov}_P(\theta)/N$; taking the trace,
    \begin{equation}
        \mathbb E\big[\|\bar\theta_w-\mu_P\|^2\big]
        =\operatorname{tr}\operatorname{Cov}(\bar\theta_w)
        =\frac{\operatorname{tr}\operatorname{Cov}_P(\theta)}{N},
    \end{equation}
    which is finite by the second-moment assumption. Since $x\mapsto\sqrt{x}$ is concave, Jensen's inequality gives $\mathbb E\big[\|\bar\theta_w-\mu_P\|\big]\le\sqrt{\mathbb E\big[\|\bar\theta_w-\mu_P\|^2\big]}=\sqrt{\operatorname{tr}\operatorname{Cov}_P(\theta)/N}$.
\end{proof}

\begin{figure}[H]
    \centering
    \includegraphics[scale=0.45]{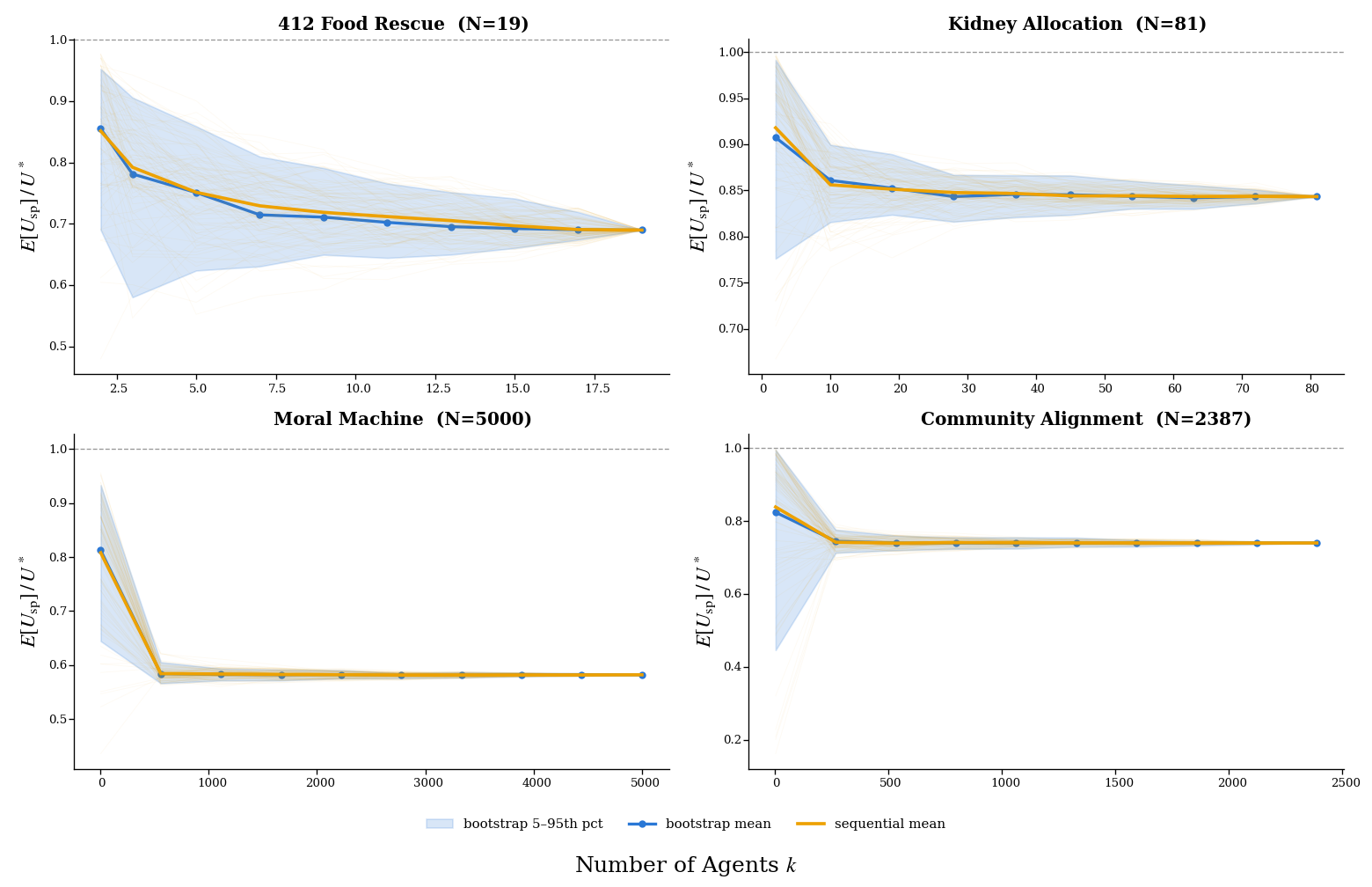}
    \caption{Empirical finite-sample stability of the strategyproof mechanism's welfare; Proposition~\ref{prop:population} gives the corresponding theoretical population guarantee under the assumption that annotators are representative of the target population. Expected welfare for four datasets. Bootstrap means and standard deviations were calculated by bootstrapping $l$ agents 100 times, for each subset size $l$, and calculating the welfare of the strategyproof mechanism. Sequential addition denotes the procedure where 100 permutations of agents were generated, and welfare was calculated for each coalition created by adding a new agent.}
    \label{fig:expectedwelfare}
\end{figure}

\begin{figure}[H]
    \centering
    \includegraphics[scale=0.35]{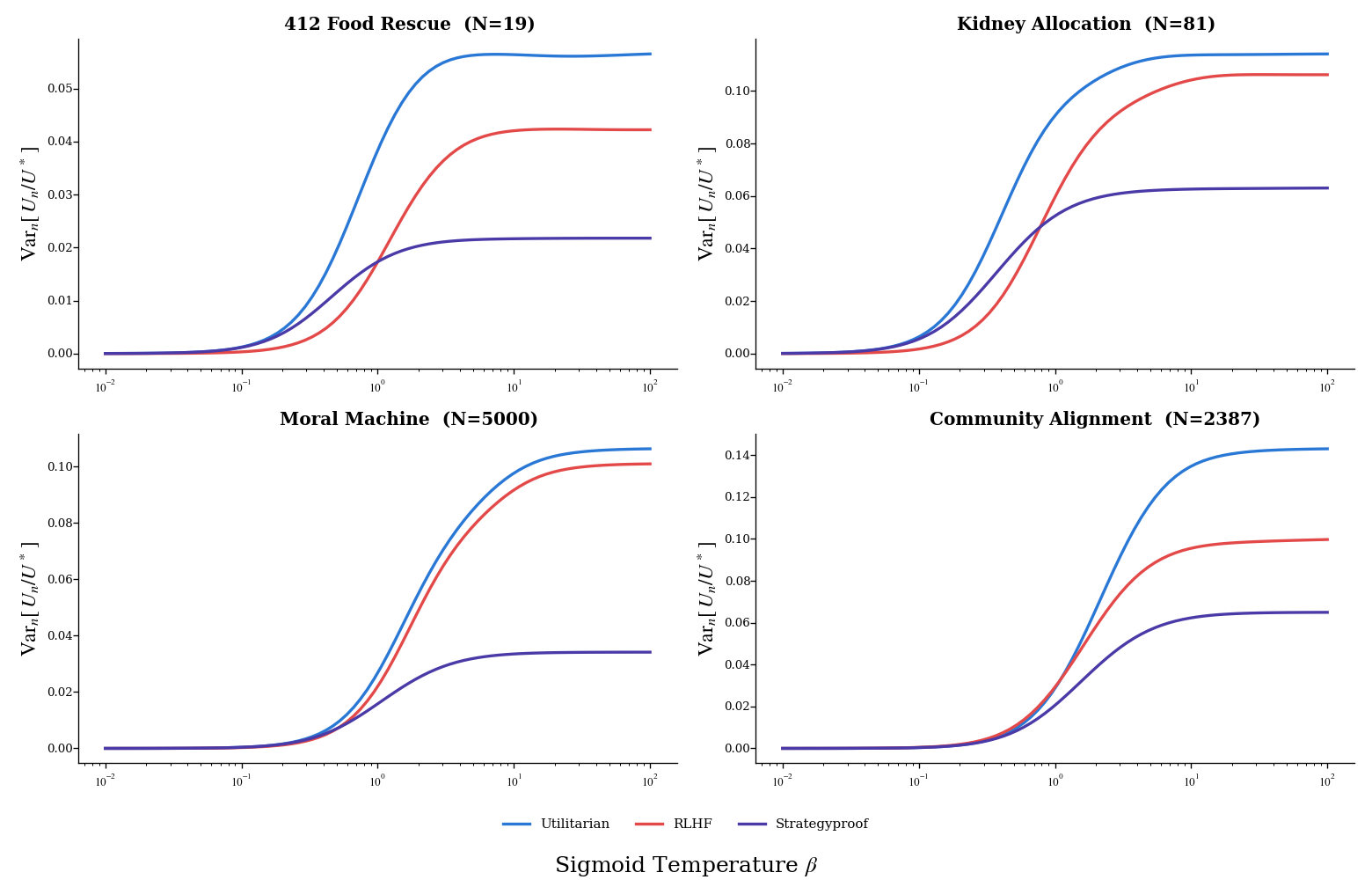}
    \caption{Welfare variance over four datasets. The strategyproof mechanism leads to the least welfare dispersion, while RLHF follows as a distant second for larger values of $\beta$.}
    \label{fig:welfarevar}
\end{figure}

\begin{figure}[H]
    \centering
    \includegraphics[scale=0.4]{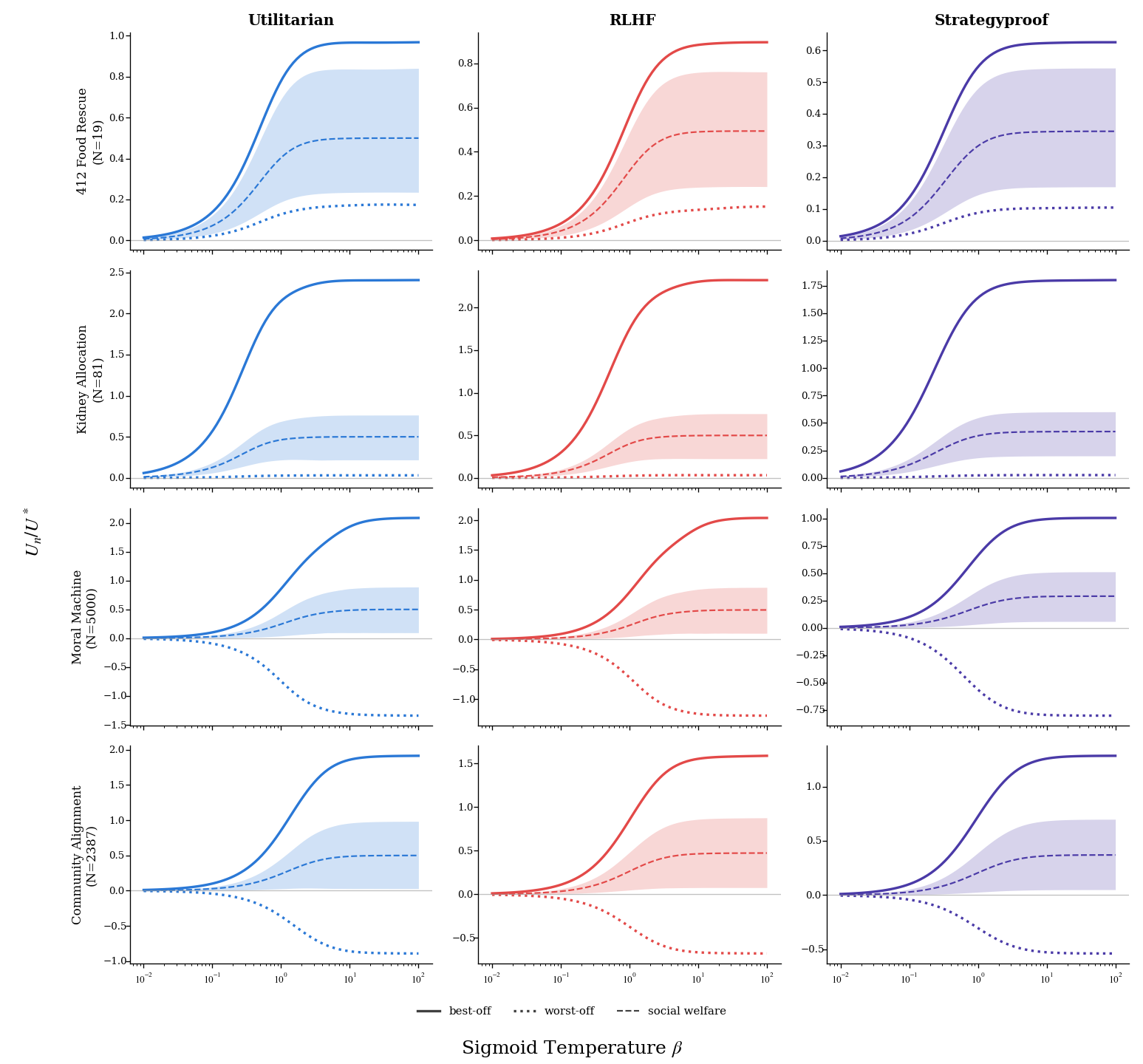}
    \caption{Welfare of the best- and worst-off agents in four datasets for three model parameters. }
    \label{fig:peragentwelfare}
\end{figure}

\bibliographystyle{plainnat}

\end{document}